\documentclass{article}

\usepackage[utf8]{inputenc} 
\usepackage[T1]{fontenc}    
\usepackage[
  colorlinks=true,
  linkcolor=black,
  citecolor=black,
  urlcolor=black
]{hyperref}                 
\usepackage{url}            
\usepackage{booktabs}       
\usepackage{amsfonts, dsfont}       
\usepackage{nicefrac}       
\usepackage{microtype}      
\usepackage{graphicx}
\usepackage{amsmath,amssymb,amsthm,mathtools,fullpage}
\usepackage{float} 

\usepackage{arxiv}

\usepackage{tikz}
\usetikzlibrary{arrows.meta, positioning, fit, backgrounds}
\usepackage{pgfplots}
\usepgfplotslibrary{fillbetween}
\pgfplotsset{compat=1.18}

\usepackage[sorting=nty, style=ieee, backend=biber, citestyle=numeric-comp]{biblatex} 
\newcommand{\N}{\mathbb{N}}
\newcommand{\R}{\mathbb{R}}

\newcommand{\e}{\varepsilon}
\newcommand{\arccosh}{\operatorname{arccosh}}

\newtheorem{theorem}{Theorem}
\newtheorem{lemma}[theorem]{Lemma}
\newtheorem{claim}[theorem]{Claim}
\newtheorem{corollary}[theorem]{Corollary}
\newtheorem{remark}[theorem]{Remark}

\theoremstyle{definition}

\newtheorem{definition}[theorem]{Definition}

\usepackage{xcolor}

\begin{document}

\title{Limitations of~Learning Tanh Neural Networks\\with Finite Precision}

\author{
Philipp Grohs$^{1,2}$ and~Matěj Trödler$^{1}$
}

\newcommand{\affiliations}{
$^{1}$ Faculty of~Mathematics, University of~Vienna\\
<firstname>.<lastname>@univie.ac.at\\
$^{2}$ RICAM, Austrian Academy of~Sciences\\
<firstname>.<lastname>@oeaw.ac.at
}

\maketitle

\begin{abstract}
    We investigate limitations of~learning $\tanh$ neural networks from~point evaluations under finite-precision computations and~$L^p$ accuracy guarantees, building on~Berner, Grohs, and~Voigtl\"ander (2023). Our approach is based on~a~novel construction of~sharply localized bump functions via iterated $\tanh$ activations. Using this mechanism, we show that, in~a~finite-precision setting, no adaptive randomized algorithm based on $m$ samples can achieve a convergence rate higher than the Monte Carlo rate $O(m^{-1/p})$ in the $L^p$ norm, unless the sampling budget grows exponentially with the~size of the network parameters and~architecture. The~results reveal fundamental limitations imposed by~finite precision on~the~learnability of~classes containing localized bump functions, extending previous results for~ReLU networks to~the~$\tanh$ setting.
\end{abstract}
%


\section{Introduction}
Since neural networks play a~crucial role in~modern machine learning, the~theoretical understanding of~their complexity, expressivity, and~limitations has become a~central topic of~mathematical research. By~now, there exist deep theoretical results showing that neural networks possess enormous expressive power and~can approximate continuous functions with arbitrary accuracy; see, for~example,~\cite{pinkus1999approximation,grohs2023proof,elbrachter2021deep,bolcskei2019optimal}. In~particular, universal approximation results for~feedforward neural networks with sigmoidal activation functions\footnote{The classical sigmoid activation is merely a~shifted and~rescaled version of~$\tanh$.} are established; see, e.g.,~\cite{barron2002universal,de2021approximation}. However, practical learning problems are substantially more challenging. In~realistic applications, the~target function is generally unknown, and~one typically has access only to~partial information about it, most often in~the~form of~finitely many sampled values.

More precisely, given a~model class $U\subset C([0,1]^d,\R)$, in~supervised learning one aims to~recover an~unknown function $u\in U$ from~its evaluations at~finitely many sample points $x_1,\dots,x_m\in[0,1]^d$ via  an~algorithm $A$, which, using only the~sampled values $u(x_1),\dots,u(x_m)$, constructs an~approximation satisfying $A(u)\approx u$.
If~the~reconstruction error is measured in~an~$L^2$ norm, statistical learning theory yields probabilistic bounds in~terms of~the~number of~noisy samples~\cite{cucker2002mathematical,bousquet2003introduction,berner2020analysis,lugosi2019risk}, typically giving the~Monte Carlo error rate $m^{-1/2}$, with~implicit constants depending on~the~complexity of~$U$. For~noiseless samples, one would typically expect much higher reconstruction rates; for~instance, when~$U$ consists of~finite element functions, a~fixed number of~samples can already allow exact reconstruction of~any~$u \in U$~\cite{ern2004theory}.

To clarify the~optimal rate for~classes of~neural networks, we are interested in~deriving lower bounds on~the~error in~terms of~the~number of~samples. More specifically, we ask the~following question: how many samples $m$ are necessary to~guarantee that the~best possible algorithm can approximate every function $u\in U$ to~accuracy $\e$ based solely on~sampled function values?

In~\cite{berner2023learning,grohs2024proof,grohs2023sobolev}, the~authors investigated this question in~the~ReLU setting, in~particular the~limitations of~learning classes containing ReLU networks. They showed that, as the~architecture size increases, the~learning error grows exponentially, while it decays in~terms of~the~number of~samples at~the~rate $m^{-1/d-1/p}$, where $d$ denotes the~input dimension and~$p$ the~exponent in~the~$L^p$ error metric. Their result is remarkable in~that it is entirely independent of~the~computational running time of~the~learning algorithm: it demonstrates that even if the~optimization problems associated with learning such function classes were computationally feasible, the~required number of~samples alone renders efficient learning impossible.
For further motivation, discussion, and~a comprehensive survey of this line of work, we refer to~\cite{berner2023learning}.

Motivated by~this work, we investigate similar phenomena for~neural networks with the~widely used hyperbolic tangent activation function $\tanh$. This activation appears prominently in~recurrent architectures such as RNNs~\cite{sherstinsky2020fundamentals} and~their enhanced variants including LSTM~\cite{hochreiter1997long} and~GRU~\cite{cho2014properties}. Moreover, $\tanh$ is frequently employed in~physics-informed neural networks (PINNs)~\cite{raissi2019physics}, where its smoothness ($\tanh \in C^\infty(\mathbb{R})$), odd symmetry, and~boundedness provide convenient analytical and~numerical properties~\cite{girault2024approximation}. The~activation is also closely related to~modern activations such as GELU~\cite{hendrycks2016gaussian}, and~thus indirectly connected to~architectures widely used in~NLP and~LLMs, including BERT~\cite{devlin2019bert} and~foundational GPT models~\cite{radford2018improving}. More generally, smooth sigmoidal activations play an~important role in~many neural-network architectures due to~these properties, making them attractive both in~applications and~theoretical analysis.

Unfortunately, these advantages also introduce substantial analytical difficulties.
In~particular, repeated compositions of~smooth sigmoidal activations rapidly lead to~highly intricate functional expressions involving nonlinear interactions of~exponential terms. Moreover, the~qualitative behavior of~such compositions strongly depends on~the~underlying weight matrices and~biases: depending on~the~parameter regime, the~resulting functions may resemble sharp sign-type transitions, nearly vanishing functions, or~highly nonlinear interpolations between these extremes, often exhibiting exponential sensitivity with respect to~the~network parameters. 
An additional complication stems from~the~fact that sigmoidal activations are real-analytic. This has an~important practical consequence: for~nonzero weights and~nontrivial inputs, the~output is typically nonzero as well. As a~result, it is generally not possible to~enforce exact vanishing of~the~network output on~sets of~positive measure. This is a~key difference compared to~the~approach of~\cite{berner2023learning}, where localized bump constructions supported on~small cubes are possible and~play an~essential role. In~the~sigmoidal setting, such compactly supported constructions are no longer available in~exact form. 
Moreover, the~results of~\cite{berner2023learning} crucially rely on~the~homogeneity property of~the~ReLU activation $\rho(x)=\max\{0,x\}$, namely $
\rho(\lambda x)=\lambda \rho(x)$ for~$\lambda\in (0,\infty) $,
which provides a~powerful structural tool in~theoretical analysis. In~the~setting of~nonlinear sigmoidal activations, however, this mechanism is no longer available.
To summarize: in~order to~study the~optimal sampling complexity for~neural network classes with sigmoidal activation functions, different approaches and~perspectives become necessary. 

A main novel perspective of~this work is therefore the~incorporation of~finite precision arithmetic. Since neural networks are evaluated on~finite-precision machines, values below some threshold $\e_p>0$ are effectively indistinguishable from~zero, which we take into account. This enables us to~construct localized bumps of~positive measure whose tails decay exponentially fast outside a~small cube and~fall below machine precision, where they are effectively treated as zero by~the~computational device. This allows us to~establish results analogous to~those in~\cite{berner2023learning}, now relying solely on~the~finite-precision assumption, which forms a~crucial ingredient.

Another technical contribution concerns the~mathematical mechanism used to~construct localized bump functions which leverages the~fixed-point structure of~the~$\tanh$ activation, in~particular with the~existence of~nontrivial solutions $x_\star$ of~$\tanh(ax_\star)=x_\star.$ Such fixed-point phenomena have classical roots in~Brouwer's and~Banach's fixed-point theorems and~have been extensively studied in~various fields. Fixed-point properties of~neural networks themselves have also recently attracted attention, see e.g.~\cite{davydov2024non, jafarpour2021robust}. In~contrast to~these works, our analysis exploits the~fixed-point behavior directly at~the~level of~compositions of~network layers. This mechanism forms one of~the~key ingredients of~the~proof and~is investigated in~detail mainly in~Subsection~\ref{subsection:properties_of_sigma_n}.

The main result of~this paper proves that hardness results as in~\cite{berner2023learning} persist also in~the~case of~$\tanh$ networks. In~the~simplified theorem below, which constitutes a~condensed version of~the~main result of~this paper, we prove that under the~assumption of~finite precision and~relatively mild additional conditions induced by it, the~necessary number of~samples for~learning the~class $U$ exhibits a~dependence similar to~that obtained in~the~ReLU setting. In~particular, the~required number of~samples grows exponentially with respect to~the~network architecture.

\begin{theorem}[Simplified main result]\label{thm:main_result_simp}
Let $p,q\in[1,\infty]$.
Suppose that $U$ contains all $\tanh$ neural networks with input dimension $d$, depth $L$, width $B$, and~coefficients bounded by~$c$ in~$\ell^q$ norm.
Assume that
$\tilde c:=B^{1-2/q}c>1$, $B\ge 2d$,
and~that $s,j\in\mathbb N$ with $s\le d$, $7\le L$ and~$j\le L-6$ are chosen depending on~parameters above.
Assume that algorithms operate under finite precision $\e_p>0$.
Then for~all
\[
m
\lesssim
c^{\frac s2(5+j)}
B^{\frac s2\left(
1+j-\frac1q(2j+7)
\right)},
\]
every algorithm using $m$ point samples for~reconstructing $U$ incurs $L^p$ approximation error $\e$ at~least
\[
\e
\ge
c\,
\frac{\sqrt{\tilde c^2-1}}{4\tilde c}
\left(
\frac{3\tanh\big(B^{-\frac1q}c\big)}
{5B^{-\frac 1q}c\,
2^{1+\frac2s}\sqrt{s}}
\right)^{\frac{s}{p}}
m^{-\frac1p}.
\]
\end{theorem}

The~complete theorem, including all explicit parameter dependencies and~technical assumptions, is stated in~Theorem~\ref{thm:main_result_full_statement} and~proved in~Subsection~\ref{subsection:final_lower_bound}. The~simplified form above highlights the~principal qualitative behavior of~the~bound. To~illustrate its implications, we show that even for~networks of~moderate size, learning to~uniform accuracy becomes computationally infeasible in~the~finite-precision setting. Consider $p=q=\infty$, $c=2$, $d=15$, $B=45$, $L=12$, and~machine precision $\e_p=10^{-69}$. In~this regime, the~assumptions are satisfied and~the~admissible sample size becomes $m \le 10^{71}$, while the~resulting lower bound yields $\e \ge 0.49$, which makes the~uniform recovery infeasible.\label{numerical_example}

Moreover, the~result reveals a~crucial intrinsic instability: for~every algorithm using at~most $10^{71}$ samples, there exist neural networks whose sampled values differ by~an almost negligible perturbation $2\cdot 10^{-69}$, while their $L^\infty$-distance remains still of~order $0.5$. Hence, even $10^{71}$ samples do not determine the~underlying function in~a~stable way. In~other words, even if highly accurate reconstruction methods existed, they would necessarily be extremely unstable. Similar phenomena appear for~analytic interpolation problems. For~instance, \cite{platte2011impossibility} shows that highly accurate reconstruction of~analytic functions from~equispaced interpolation points is necessarily unstable, whereas suitable choices of~interpolation nodes, such as Chebyshev points, permit both stability and~accuracy. By contrast, our result implies a~substantially stronger instability phenomenon for~$\tanh$ neural networks: regardless of~how the~sampling points are chosen, no reconstruction procedure can be simultaneously accurate and~stable within the~stated sampling regime.

For~comparison with~the~ReLU setting and~$q\ge2$, the~result of~\cite[Theorem 2.2]{berner2023learning} yields the~lower bound
\[
\e
\ge
\mathrm{const}\cdot
\frac{c^LB^{(L-1)(1-\frac2q)}}{(32s)^{1+\frac{s}{p}}}
m^{-\frac1p-\frac1s}.
\]
The structure of~the~two estimates is substantially different. 
First of~all,  the~characteristic factor $m^{-1/d}$ no longer appears in~our result.
Moreover, in~the~ReLU setting, the~lower bound grows exponentially with respect to~the~network architecture parameters. This behavior originates from~the~definition of~the~ReLU activation, which preserves and~amplifies large values through compositions. In~contrast, for~$\tanh$ networks the~amplitude of~the~lower bound remains uniformly bounded, since $|\tanh|\le1$. On~the~other hand, the~exponential dependence of~the~admissible number of~samples $m$ on~the~architecture parameters remains present also in~the~sigmoidal setting.

Comparing the~two estimates more closely, one observes that after factoring out the~common term $m^{-1/p}$, the~ReLU bound contains the~additional factor $c^LB^{(L-1)(1-\frac2q)}\,s^{-1}\,m^{-\frac1s}$,
which suggests that the~admissible number of~samples in~the~ReLU setting may asymptotically be roughly quadratically larger than in~the~$\tanh$ setting of~Theorem~\ref{thm:main_result_simp}. On~the~other hand, the~remaining factor $(32s)^{-s/p}$ in~the~ReLU estimate may become smaller than the~corresponding sigmoidal term involving $B^{s/(qp)}(c\sqrt{s})^{-s/p}$. Moreover, an~important distinction appears in~the~case $p=\infty$: while the~ReLU lower bound still decays with increasing $m$, the~corresponding lower bound in~the~$\tanh$ setting remains bounded from~below by~a~positive constant.

For a~more precise non-asymptotic comparison, one has to~use the~full statement of~our lower bound in~Theorem~\ref{thm:main_result_full_statement}. To~facilitate numerical verification of~the~involved expressions, we have developed an~interactive application that allows users to~specify network parameters (such as architecture, regularization, machine precision, and~number of~samples) and~to~compute the~resulting worst-case approximation error. The~application is publicly available at
\begin{center}
    \url{https://learning-tanh-nn.streamlit.app/}
\end{center}
and automatically checks whether all assumptions of~the~theorem are satisfied. The~user interface is illustrated in~Figure~\ref{fig:app_screenshot}.

\begin{figure}[ht]
    \centering
    \includegraphics[width=\linewidth]{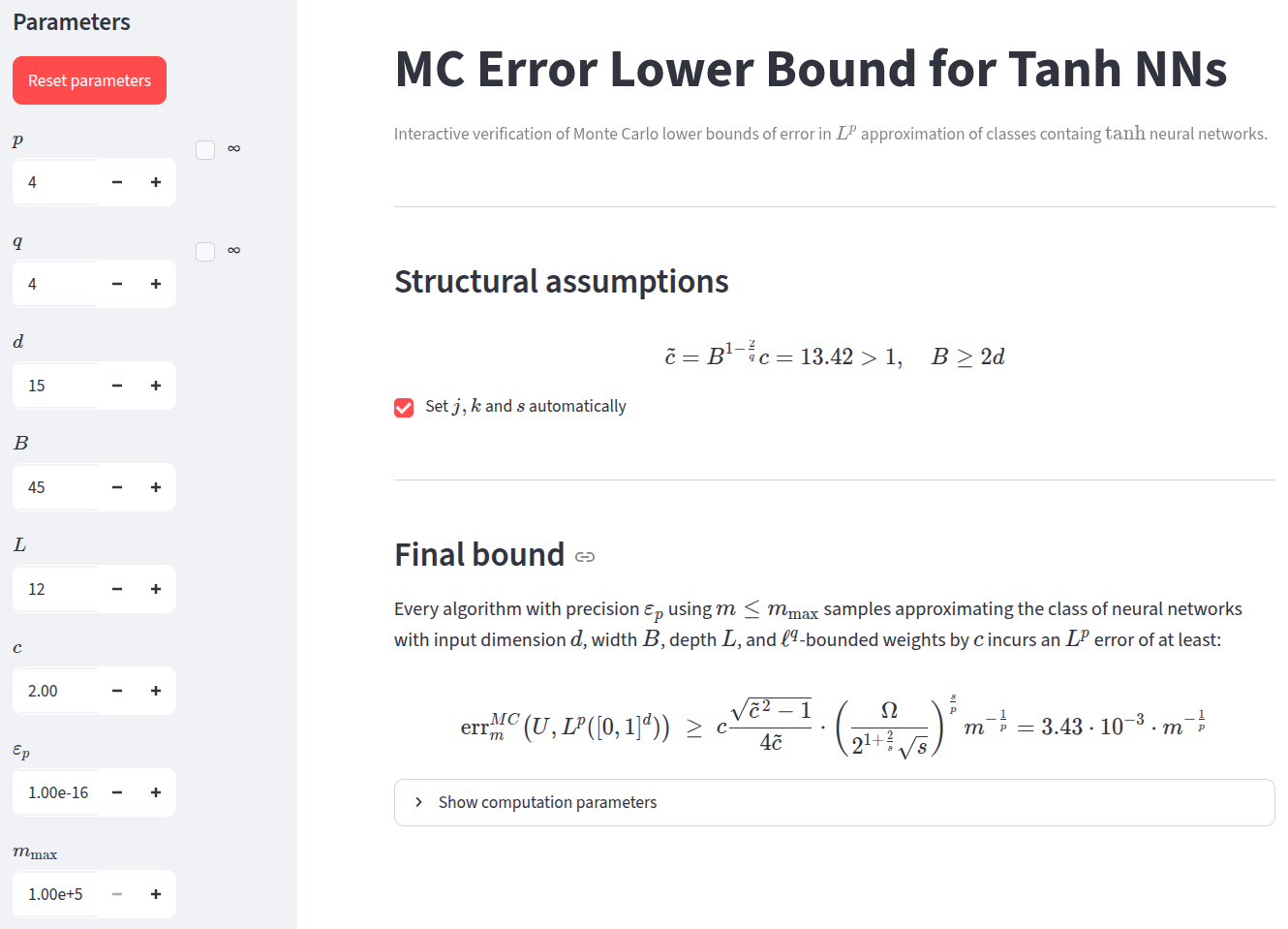}
    \caption{Screenshot of~the~application for~numerical verification of~the~main result in~Theorem~\ref{thm:main_result_full_statement}.}
    \label{fig:app_screenshot}
\end{figure}

\section{Preliminaries}

We now introduce the~notation used throughout the~text, largely following~\cite{berner2023learning} where appropriate.

\subsection{Notation}
For $d\in\N$, we denote $C([0,1]^d)$ the~space of~continuous functions $f\colon [0,1]^d\to\R$, and~$[d] = \{1,\dots,d\}$. For~a~matrix $W\in\R^{n \times k}$ and~$q\in[1,\infty)$, its $\ell^q$ norm is $\|W\|_{\ell^q}
=
\big(\sum_{i,j}|W_{i,j}|^q\big)^{1/q}.$
For $q=\infty$, we set $\|W\|_{\ell^q}
=
\max_{i,j}|W_{i,j}|.$ For~vectors $b\in\R^{n}$, we identify $b$ with a~matrix in~$\R^{n\times1}$ and~use the~same definition.
Finally, for~$n,k\in\N$, we write
\[
1_{n\times k}
:=
\begin{pmatrix}
    1&\dots&1\\
    \vdots&\ddots&\vdots \\
    1 & \dots & 1
\end{pmatrix} \in \R^{n\times k}.
\]

\begin{table}[ht]
\centering
\renewcommand{\arraystretch}{1.25}
\setlength{\tabcolsep}{6pt}
\begin{tabular}{@{} p{3.8cm} l @{}}
\toprule
\textbf{Symbol} & \textbf{Definition} \\
\midrule

$\sigma_a(x)$ &
$\displaystyle \sigma_a(x):=\tanh(ax), \qquad \sigma:=\sigma_1$
\\

$\sigma_a^{\,n}(x)$ &
$\displaystyle
\sigma_a^{\,0}(x)=x,\quad
\sigma_a^{\,n}(x)=\sigma_a\!\left(\sigma_a^{\,n-1}(x)\right),
\quad n\in\mathbb N
$
\\
$x_\star(a)$ &
positive fixed point satisfying $\sigma_a(\pm x_\star)= \pm x_\star$
\\

$\bar{x}(a)$ &
positive critical point satisfying $\sigma_a'(\pm \bar{x}) = 1$
\\

$\tau(a)$ &
chosen such that $\sigma_a(\bar{x}-\tau) > \bar{x}+\tau$
\\

$I(a,\tau)$ &
$\displaystyle I := (-\bar{x}+\tau, \bar{x}-\tau)
$ \\

$J(a,\tau)$ &
$\displaystyle J := (-\infty, -\bar{x} - \tau) \cup (\bar{x}+\tau, \infty)
$ \\

$n_0(x) = n_0(x\,|\,a,\tau)$ &
first exit time:
$\displaystyle n_0(x) := \sup\{n\in\N : \sigma_a^{\,n}(x)\in I\}$
\\

$\xi(a,\tau)$ &
$\displaystyle \xi := \inf_{x \in I} \sigma_a'(x) > 1
$ \\

$\eta(a,\tau)$ &
$\displaystyle \eta := \sup_{x \in J} \sigma_a'(x) < 1
$ \\

$\bar{\eta}(a,\tau)$ &
$\displaystyle \bar{\eta} := \sup_{x \in J} \sigma_a'(\sigma_a(x)) < 1
$ \\

$\mu(a)$ &
$\displaystyle \mu := \sigma_a'(x_\star) < 1
$ \\

$\xi_n(a,\tau,\kappa)$ &
$\displaystyle \xi_n := \xi^{-\lceil \kappa n \rceil}(\bar{x} - \tau)
$ \\

$\e_n(a,\tau,\kappa)$ &
$\displaystyle \e_n := \eta\, \bar{\eta}^{\lfloor (1-\kappa)n \rfloor - 3}
$ \\

$\rho(a)$ &
$\displaystyle 
\rho(a) := \frac{1}{\sqrt{a^2 - \frac{4a^2}{(a+1)^2}\,\arccosh^2(\sqrt{a})}+1}
$ \\

$\chi(a)$ & $\displaystyle \chi(a) := \frac{1}{a+1}$ \\

$\pi(a)$ &
$\displaystyle 
\pi(a) := \left(\sqrt{a} + \sqrt{a-1}\right)^{4(1-\rho(a))}
$ \\

$f(y) = f(y\,|\,B^{-1/q}, c, b)$ & $\displaystyle f(y) := \tanh\left(B^{-1/q}\big((c-b)y + b\big)\right) - \tanh\left(B^{-1/q}\big((c-b)y - b\big)\right)$ \\

\bottomrule
\end{tabular}
\caption{Definitions of~frequently used quantities and~parameters (strictly positive where applicable). Throughout, we often omit repeated parameter dependencies when they are clear from~the~context. Note that $\sigma_a^{\,n}$ denotes the~$n$-fold iterate of~$\sigma_a$ under composition, not a~power.}
\label{tab:definitions}
\end{table}

\subsection{Adaptive (randomized) algorithms}

We adopt the~framework of~adaptive algorithms from~\cite{berner2023learning}, allowing for~a~general quantization map $Q$ (the setting in~\cite{berner2023learning} corresponds to~$Q=\operatorname{Id}$). Such algorithms have access only to~quantized function values at~sampled points. The~term \emph{adaptive} refers to~the~fact that each new sampling point may depend on~all previously observed samples and~quantized function values. In~the~randomized setting, the~algorithm may additionally employ internal randomness and~use a~random number of~samples, subject only to~a~prescribed bound on~the~expected sampling cost.

\begin{definition}[Adaptive algorithms]
Let $d,m\in\N$, let $Y$ be a~Banach space, let $U\subset C([0,1]^d)\cap Y$, and~let $Q:\R\to\R$ be a~precision mapping. We say that a~mapping $A\colon U\to Y$ is an~adaptive deterministic method using $m$ point samples and~working under precision mapping $Q$ if there exist sampling rules selecting the~sampling points adaptively based on~previously acquired function values,
\[
f_1\in [0,1]^d,
\qquad
f_i\colon \big([0,1]^d\big)^{i-1}\times \big(Q(\R)\big)^{i-1}\to [0,1]^d,
\quad i\in\{2,\dots,m\},
\]
and a~reconstruction map
\[
F\colon \big([0,1]^d\big)^m\times \big(Q(\R)\big)^m\to Y,
\]
such that for~every $u\in U$, the~sampling points $\mathbf{x}(u)=(x_1,\dots,x_m)\subset [0,1]^d$ are chosen recursively by
\[
x_1=f_1,
\qquad
x_i
=
f_i\Big(
x_1,\dots,x_{i-1},
Q\big(u(x_1)\big),\dots,Q\big(u(x_{i-1})\big)
\Big),
\quad i\in\{2,\dots,m\},
\]
and
\[
A(u)
=
F\Big(
x_1,\dots,x_m,
Q\big(u(x_1)\big),\dots,Q\big(u(x_m)\big)
\Big).
\]
The set of~all adaptive deterministic methods using $m$ samples with mapping $Q$ is denoted by~$\operatorname{Alg}_{m,Q}(U,Y)$.
\end{definition}

\begin{definition}[Adaptive randomized algorithms]
Let $(\Omega,\mathcal F,\mathbb P)$ be a~probability space, let $Y$ be a~Banach space, let $U\subset C([0,1]^d)\cap Y$, $Q\colon  \R \to \R$ and~let $m\in\N$. A~pair
\[
(\mathbf A,\mathbf m),
\qquad
\mathbf A=(A_\omega)_{\omega\in\Omega},
\]
is called an~adaptive randomized method under the~precision mapping $Q$ using $m$ samples on~average if the~following conditions are satisfied:
\begin{enumerate}
    \item the~mapping $\mathbf m\colon \Omega\to\N$ is measurable and~$\mathbb E(\mathbf m)\le m$,

    \item for~every $u\in U$, the~mapping $\omega\mapsto A_\omega(u)$ is measurable with respect to~the~Borel $\sigma$-algebra on~$Y$,

    \item for~every $\omega\in\Omega$, the~mapping $A_\omega$ is an~adaptive deterministic method using $\mathbf m(\omega)$ point samples under mapping $Q$.
\end{enumerate}
The collection of~all such methods is denoted by
$\operatorname{Alg}^{MC}_{m,Q}(U,Y)$.

Moreover, the~optimal randomized error is defined by
\[
\operatorname{err}^{MC}_{m,Q}(U,Y)
:=
\inf_{(\mathbf A,\mathbf m)\in \operatorname{Alg}^{MC}_{m,Q}(U,Y)}
\;
\sup_{u\in U}
\;
\mathbb E\!\left(
\|u-A_\omega(u)\|_Y
\right).
\]
\end{definition}

\subsection{Quantization map}

Our main result is formulated in~a~finite-precision setting. For~our purposes, the~only essential property of~a~quantization scheme is that sufficiently small values are rounded to~zero. For~an~overview of~various quantization techniques, we refer, for~instance, to~\cite{gray2002quantization}.

\begin{definition}[Quantizer]
Fix $\e>0$. A~finite-precision quantizer is a~map
\[
Q_\e:\R\to X_{Q_\e},
\]
where $X_{Q_\e}\subset\R$ is a~(possibly finite) set such that
\[
Q_\e(x)=0
\qquad\text{whenever } |x|\le\e.
\]
\end{definition}

\subsection{Neural network classes}

We consider feedforward neural networks with hyperbolic tangent activation $\sigma=\tanh$, denoted by~$\Phi$. The~input dimension is $d\in\N$, the~depth (number of~layers) is $L\in\N$, and~the~layer widths are $N_0,\dots,N_L\in\N$. In~the~special case
\[
N_1=\dots=N_L=:B,
\]
we call $B\in\N$ the~common width and~introduce the~width ratio $r:=B/d$.
For affine maps $\phi^i\colon \R^{N_{i-1}} \to \R^{N_i}$ defined as $\phi^i(x)=W_i x+b_i$ for~$i\in[L]$, the~realization of~a~network is given by
\[
R(\Phi)=\phi^L \circ \sigma \circ \phi^{L-1} \circ \cdots \circ \sigma \circ \phi^1
\in C(\R^d,\R^{N_L}).
\]
Let $q\in[1,\infty]$. The~$\ell^q$-norm of~a~network $\Phi$ is defined by
\[
\|\Phi\|_{\ell^q}
=
\max_{i\in [L]}
\left\{
\|W_i\|_{\ell^q},
\|b_i\|_{\ell^q}
\right\}.
\]
We will frequently use the~parameters
\[
c:=\|\Phi\|_{\ell^q},\qquad
\tilde c:=B^{1-\frac2q}c.
\]
Further notation is summarized in~Table~\ref{tab:definitions}.

\begin{definition}[Class of~neural networks]
For $d,L\in\N$ with $N_0=d$, widths $N_1,\dots,N_L\in\N$, $q\in[1,\infty]$, and~$c>0$, we define the~class of~$\tanh$ neural networks by
\[ 
\mathcal{H}_{(N_0,\dots,N_L),c}^q 
:= 
\left\{ R(\Phi)\;:\; \Phi = \bigl((W_\ell,b_\ell)\bigr)_{\ell=1}^L,\; W_\ell \in \R^{N_{\ell-1}\times N_\ell},\; b_\ell \in \R^{N_\ell},\; \|\Phi\|_{\ell^q} \le c \right\}. 
\]
\end{definition}

\section{Main result}

We now state the~main result in~full form. It provides a~lower bound on~the~sampling complexity for~learning $\tanh$ neural networks under finite-precision arithmetic $Q_{\e_p}$. The~proof relies on~constructing networks that exhibit a~separation between a~small localized mass region and~a~tail that lies below the~quantization threshold $\e_p$. The~architecture is designed so that part of~the~depth ($k$ layers) enforces decay below this threshold, while the~remaining depth $j$ is used to~maximize the~admissible sampling budget. The~parameter $s$ captures an~effective dimension reduction arising from~a~restricted active subspace, which improves the~resulting lower bound on~the~error.

\begin{theorem}[Main result]\label{thm:main_result_full_statement}
Let $p,q\in[1,\infty]$, $d,B,L,j,k\in\N$ and~$c>0$ satisfy $B\ge 2d$, $k\geq3$, $j+k=L-3$ and~$\tilde c :=B^{1-\frac{2}{q}}c>1$.
Define
\[
\Theta := \frac{9\,\tanh\left(B^{-1/q}\tanh\left(\frac c2\right)\right)}{50\,\cosh^2\left(B^{-1/q}\tanh\left(\frac c2\right)\right)}
\qquad 
\text{and}
\qquad
\Omega := \frac35 \cdot \frac{\tanh\left(B^{-1/q}\left(c-\tanh\left(\frac c2\right)\right)\right)}{B^{-1/q}\left(c-\tanh\left(\frac c2\right)\right)}.
\]

Let 
\[
\mathcal{H}_{(d,B,\dots,B,1),c}^q \subset U \subset C([0,1]^d),
\]
where the~width $B$ appears $L-1$ times in~$(d,B,\dots,B,1)$. 
Assume that computations are performed under a~quantizer $Q_{\e_p}$ with precision $\e_p>0$, and~that $q,c,B$ are sufficiently large so that
\begin{equation}\label{eq:assumption_on_k_num_of_lyers}
3
+ \frac{
\ln\!\left( \frac{4c \tilde c}{\e_p \pi(\tilde c)\left(1+\pi(\tilde c)^{-1}\right)^2} \right)
}{
\ln \!\left( \frac{\cosh^2\!\left(\tilde c\frac{\pi(\tilde c)-1}{\pi(\tilde c)+1}\right)}{\tilde c} \right)
}
\le
k
\end{equation}
and
\begin{equation*}
\frac{\ln\left( \frac{5 B^{\frac3q}\arccosh\left(\sqrt{\tilde c}\right)\,\rho(\tilde c)}{c^2\tanh(2B^{-1/q}\tanh(\frac c2))\tanh^2\left[B^{-1/q}(c - \tanh(\frac c2))\right]} \right)}
{\ln\left(  \frac
{\tilde c}
{\cosh^2\left[2\arccosh(\sqrt{\tilde c})\rho(\tilde c)\right]} \right)
}
\leq
j.
\end{equation*}
Let $m_{\max}\in\N$ and~choose $s\in\N$ satisfying
\begin{equation}\label{eq:assumption_on_s_m_max}
d \;\ge\; s \;\ge\;
\frac{2\ln(4m_{\max})}{
j\,\ln\!\Big( \frac{\tilde c}{\cosh^2\!\left[2\arccosh(\sqrt{\tilde c})\,
\rho(\tilde c)\right]} \Big)
+\ln\!\left(\frac{ \Theta\,c^2(c-\tanh(c/2))^2}{16\,B^{5/q}\arccosh(\sqrt{\tilde c})\,\rho(\tilde c)}\right)
}
\end{equation}
provided that the~denominator is positive.

Then for~all $m\in\N$ with $m\le m_{\max}$,
\begin{equation}\label{eq:monte_carlo_error_lower_bound}
\operatorname{err}_{m,Q_{\e_p}}^{MC}\!\left(U, L^p([0,1]^d)\right)
\;\ge\;
c\frac{\sqrt{\tilde c^2-1}}{4\tilde c}
\cdot
\left(\frac{\Omega}{2^{1+\frac{2}{s}}\sqrt{s}}\right)^{\frac{s}{p}}
m^{-\frac{1}{p}}.
\end{equation}
\end{theorem}

\begin{figure}[!ht]
\centering
\begin{tikzpicture}[
  base/.style={
    rectangle, rounded corners=5pt,
    minimum width=8.5cm, align=center,
    font=\small, inner sep=5pt, line width=0.4pt
  },
  inputnode/.style={base,
    fill=gray!12, draw=gray!50,
    minimum width=4cm, minimum height=0.8cm
  },
  layer1/.style={base, fill=violet!12, draw=violet!50},
  layer2/.style={base, fill=teal!12, draw=teal!50},
  layermid/.style={base, fill=orange!14, draw=orange!55},
  layerout/.style={base, fill=red!10, draw=red!45,
    minimum width=8.5cm},
  outputnode/.style={base, fill=gray!12, draw=gray!50,
    minimum width=4cm, minimum height=0.8cm},
  arr/.style={-{Stealth[length=5pt,width=4pt]}, line width=0.5pt, gray!70}
]

\node[inputnode] (input)
  {$x \in [0,1]^d$};

\node[layer1, below=0.8cm of input] (L1) {%
\textbf{Layer 1 -- shifted sign-flipped $\tanh$ pairs centered at~$y_\ell$}\\[3pt]
$2d$ neurons\\[3pt]
Each coordinate: $\pm\,\sigma\Bigl(\tfrac{c-b}{B^{1/q}}(x_i-(y_\ell)_i)\pm \tfrac{b}{B^{1/q}}\Bigr)$\\[3pt]
Shifted sign-flipped $\tanh$ pairs, arranged for~averaging in~Layer~2
};

\node[layer2, below=0.55cm of L1] (L2) {%
\textbf{Layer 2 -- bump by~aggregation + threshold}\\[3pt]
$B$ neurons\\[3pt]
$W_2=\tfrac{c}{B^{1+1/q}}\mathbf{1}_{B\times 2d}$,\quad
$b_2=-\tfrac{c}{rB^{1/q}}f(\alpha M)\mathbf{1}_B$
for suitable $\alpha\in(0,1)$\\[3pt]
Averaging creates centered bump:
$\sigma\bigl(\tfrac{c}{rB^{1/q}}(\frac1d\sum_i f(x_i-(y_\ell)_i)-f(\alpha M))\bigr)$ \\[3pt]
Bias enforces $\Phi_{y_\ell, \nu}\gg0$ on~$y_\ell+\alpha[-M,M]^d$ and~$\Phi\approx0$ outside $y_\ell +[-M,M]^d$
};

\node[layermid, below=0.55cm of L2] (Lmid) {%
\textbf{Layers $3,\ldots,L{-}1$ -- sharpening}\\[3pt]
$W_i=\tfrac{c}{B^{2/q}}\mathbf{1}_{B\times B}$,\quad $b_i=0$\\[3pt]
Iterates $\sigma_{\tilde c}^{\,L-3}$ to~the~bump,\quad $\tilde c>1$\\[3pt]
Sharpens toward a~step function by~$\quad \sigma_{\tilde c}^{\,L-3}(z) \xrightarrow{L\to\infty} \operatorname{sign}(z)\cdot x_{\star, \tilde c}$
};

\node[layerout, below=0.55cm of Lmid] (LL) {%
\textbf{Layer $L$ -- output + shift}\\[3pt]
$W_L=c \,\nu \, e_1^\top,\quad b_L=c \, \nu \, x_{\star, \tilde c} \,,\quad \nu\in\{\pm1\}$\\[3pt]
Selects one neuron output, shifts by~$x_{\star, \tilde c}$ and~$\nu$ sets the~sign
};

\node[outputnode, below=0.55cm of LL] (output)
  {$\Phi_{y_\ell, \nu}(x) \colon [0,1]^d \to \R$};

\node[below left=0.7cm and 0.5cm of output, draw=teal!40, fill=teal!6, rounded corners=4pt,
      line width=0.4pt, inner sep=5pt, font=\footnotesize,
      text width=3.2cm, align=center] (outL)
{$\Phi_{y_\ell, \nu}(x)\approx 0$\\ outside $y_\ell+[-M,M]^d$};

\node[below right=0.7cm and 0.5cm of output, draw=orange!40, fill=orange!6, rounded corners=4pt,
      line width=0.4pt, inner sep=5pt, font=\footnotesize,
      text width=3.2cm, align=center] (outR)
{$\Phi_{y_\ell, \nu}(x)\ge a\, x_{\star, \tilde c}$\\ inside $y_\ell+\alpha[-M,M]^d$};

\draw[-, gray!60, thick] (output) -- (outL);
\draw[-, gray!60, thick] (output) -- (outR);

\draw[arr] (input) -- (L1)  node[midway, right] {$W_1,b_1$};
\draw[arr] (L1)  -- (L2)  node[midway, right] {$W_2,b_2$};
\draw[arr] (L2)  -- (Lmid)  node[midway, right] {$W_i,b_i$ for~$i\in\{3,\dots, L-1\}$};
\draw[arr] (Lmid)  -- (LL)  node[midway, right] {$W_L,b_L$};
\draw[arr] (LL)  -- (output);

\begin{pgfonlayer}{background}
  \node[draw=gray!40, dashed, rounded corners=5pt, line width=0.4pt,
        fit=(L1)(L2)(Lmid)(LL), inner sep=6pt] (normfit) {};
\end{pgfonlayer}
\node[above=2pt of normfit.north east,
      anchor=south east,
      font=\footnotesize, gray!60,
      align=right]
{$\sigma=\tanh$ \\ $b = \sigma(\frac{c}{2})$ \\ $\tilde c=B^{1-2/q}c$};

\end{tikzpicture}
\caption{Layer-by-layer construction of~the~bump network $\Phi_{y_\ell, \nu}$.
Layer~1 builds $d$ pairs of~shifted $\tanh$ activations with opposite signs centered at~$y_\ell$, forming coordinate-wise primitives.
Layer~2 averages these and~applies a~threshold at~$f(\alpha M)$, producing a~localized bump that is positive on~$y_\ell+\alpha[-M,M]^d$ and~negative outside $y_\ell+[-M,M]^d$.
The middle layers iteratively refine this signal via $\sigma_{\tilde c}^{L-3}$, sharpening it toward a~step function while preserving its sign structure.
The final layer selects one coordinate and~shifts by~the~fixed point $x_{\star, \tilde c}$, yielding values close to~$x_{\star, \tilde c}$ on~the~positive region and~close to~$0$ outside.
Overall, the~network produces sharply localized bumps with small support, large slope, and~exponential decay outside.}
\label{fig:network-architecture}
\end{figure}
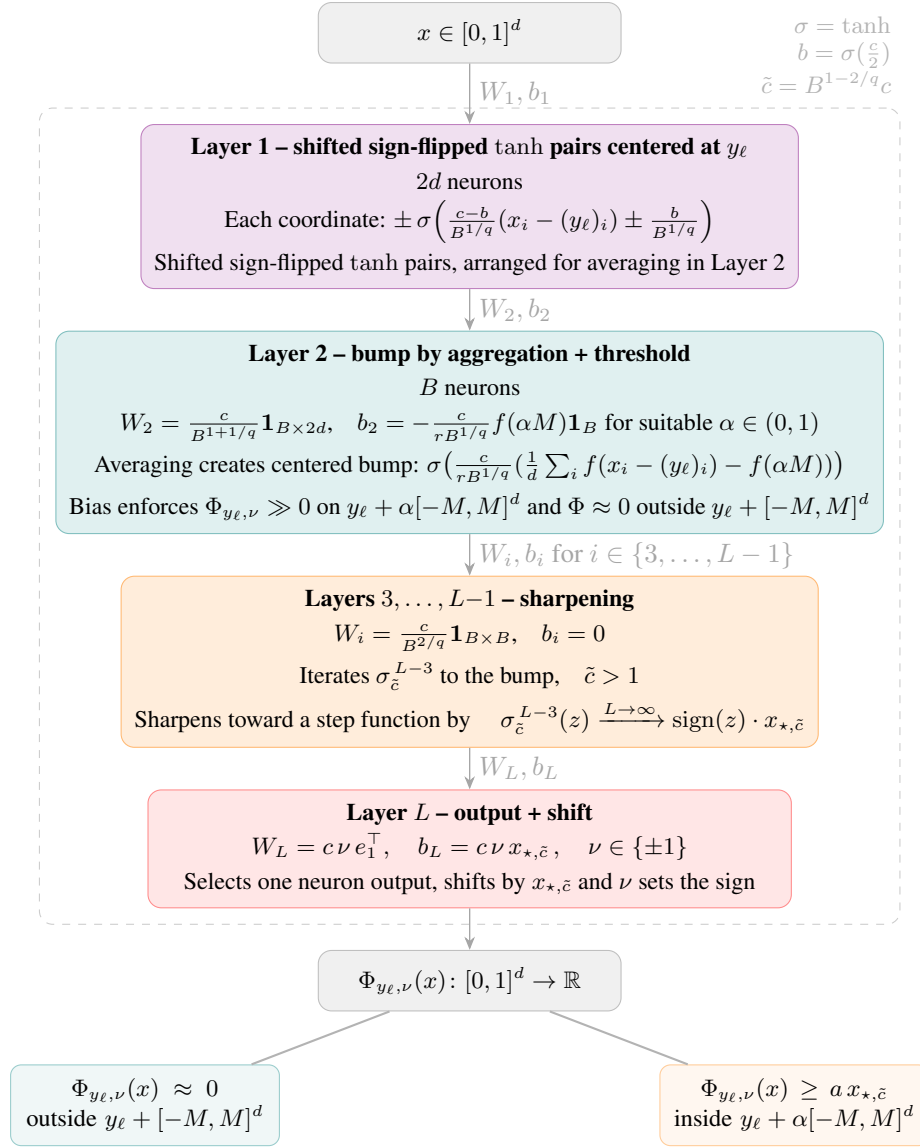

\begin{proof}[Sketch of~proof]
The full proof is provided at~the~end of~Subsection~\ref{subsection:final_lower_bound}. The~key idea is that even a~small family of~localized $\tanh$ networks cannot be reliably distinguished from~the~zero function under finite precision constraints.

We first prove Theorem~\ref{thm:existence_of_Phi_for_M>} showing that, around every point $y_\ell\in[0,1]^d$, one can build a~neural network $\Phi_{y_\ell,\nu}$ with a~localized bump structure of~sign $\nu\in\{\pm1\}$ -- for~the architecture, see Figure~\ref{fig:network-architecture}. The~construction proceeds in~two stages. First, suitable shifted pairs of~$\tanh$ activations are combined into a~coarse bump, which is positive on~$y_\ell+\alpha[-M,M]^d$ and~sufficiently negative outside $y_\ell+[-M,M]^d$ for~suitable parameters $M,\alpha\in(0,1)$.

The coarse bump is then sharpened by~iterating the~map $\sigma_a(x)=\tanh(ax)$ for~$a>1$. Repeated composition amplifies the~sign pattern and~drives the~output toward a~step-function profile; see Lemma~\ref{lemma:limit_of_gn_sgn}. The~negativity outside this cube is shown in~Lemma~\ref{lemma:condition_on_alpha_and_M}, while the~sufficiency of~the~construction is established in~Lemma~\ref{lemma:error_of_basin_of_attraction}. Consequently, after sufficiently many layers, the~network possesses two key properties: it has nontrivial $L^p$ mass concentrated on~a~small cube, with $\|\Phi_{y_\ell,\nu}\|_{L^p}\asymp c(\alpha M)^{s/p}$, where $\alpha\asymp B^{1/q}/(c\sqrt{s})$ and~$s\le d$, while outside a~slightly larger cube its values decay exponentially fast. Moreover, $M$ may be exponentially small; see Theorem~\ref{thm:existence_of_Phi_for_M>}.

By choosing the~architecture parameters, namely depth, width, and~norm constraints, suitably, this decay can be pushed below the~finite-precision threshold $\e_p$. Hence outside the~larger cube the~network becomes indistinguishable from~zero under the~quantizer, but still having positive $\|\Phi_{y_\ell, \nu}\|_{L^p}$.

Next suppose arbitrary sample points $x_1,\dots,x_m\subset[0,1]^d$. We set $M \asymp m^{-1/s}$ and~place the~centers $y_\ell$ on~a~grid that is dense only in~$s$ selected coordinates with spacing $2M$ and~constant in~the~remaining $d-s$ coordinates while preserving disjointness of~the~associated bumps due to~their support scale. A~pigeonhole argument then shows that at~least half of~the~bumps avoid all sampled points. For~such functions, the~algorithm receives exactly the~same quantized information as for~the~zero function, despite the~functions having positive $L^p$ norm.

Finally, following the~averaging argument of~\cite{berner2023learning}, this indistinguishability principle extends from~deterministic to~adaptive randomized algorithms and~yields the~stated lower bound.
\end{proof}

\begin{remark}[Generality of~the~quantizer]
The proofs only use that the~quantizer $Q_{\e_p}$ maps all values with magnitude at~most $\e_p$ to~zero. Hence, the~results apply to~any finite-precision arithmetic satisfying this threshold property; the~representation of~non-zero values is irrelevant. This includes standard IEEE~754 arithmetic, where the~effective relative precision is about $10^{-16}$, with smallest representable number of~order $2^{-1074}$.

Apart from~the~fact that the~quantitative example below Theorem~\ref{thm:main_result_simp} remains valid even for~extremely small thresholds (e.g. $\e_p=10^{-306}$ and~$m_{\max}=10^{13}$), the~construction reveals another crucial property: it is stable under shifts of~the~reference level. If parameters are chosen such that $c>x_{\star,\tilde c}$, one may replace the~last layer of~$\Phi$ by setting $b_L=c\,\nu$, thereby changing the~zero-level comparison to~a~constant shift $\nu\,(c-x_{\star,\tilde c})$. In~this regime, the~relevant scale is of~order $1$, so IEEE~754 effectively operates with relative accuracy $\approx 10^{-16}$ rather than a~tiny absolute threshold $2^{-1074}$, which leads to much larger powers in $m_{\max}$ and thus makes learning even less feasible.
\end{remark}

\begin{remark}[Choice of~$b$]
We emphasize that the~choice of~bias $b>0$ in~the~first layer is crucial. The~only constraint is $0<b<c$. A~natural choice is $b:=\tanh\!\left(c/2\right) < 1$, which clearly satisfies the~norm constraints. However, despite the~simplification in~the~subsequent computations, this choice introduces an~undesirable dependence on~$B^{1/q}/c$ in~\eqref{eq:monte_carlo_error_lower_bound}. An~alternative scaling $b\asymp cM$ could potentially remove this dependence; this approach was also used in~\cite{berner2023learning}, but it leads to~more involved estimates and~is therefore left for~future work.
\end{remark}

\begin{remark}[Effective dimension reduction]
The construction allows reducing the~effective dimension in~the~packing argument from~$d$ to~any $s\le d$ by~restricting the~grid to~an $s$-dimensional active subspace. Only $s$ coordinates contribute to~the~combinatorial growth, while the~remaining $d-s$ coordinates are kept fixed. This yields the~same separation mechanism, with the~resulting scaling governed by~$s$ instead of~$d$, leading to~the~factor $s^{-s/(2p)}$. Hence, choosing smaller $s$ strongly improves the~resulting lower bound.
\end{remark}

The proof of~the~lower bound consists of~several auxiliary lemmas and~partial results. In~the~following section, we therefore develop the~main technical ingredients in~a~structured way. In~Subsection~\ref{subsection:properties_of_sigma_n}, we study the~iterated $\tanh$ maps $\sigma_a^{\,n}$ for~$a>0$ and~analyze its convergence properties toward the~sign function. In~Subsection~\ref{subsection:bump_analysis}, we investigate the~coarse bump
\[
\frac1d\sum_{i=1}^d f(x_i - (y_\ell)_i) - f(\alpha M)
\]
and show that it can be made sufficiently small to~enforce a~controlled decay of~the~tail error below machine precision. Finally, in~Subsection~\ref{subsection:final_lower_bound}, we combine these ingredients to~derive the~final lower bound on~the~learning error.

\section{Auxiliary results and~proofs}
\subsection{Properties of~iterated tanh}\label{subsection:properties_of_sigma_n}

In this subsection we study the~iterates of~the~hyperbolic tangent map $\sigma_a(x)=\tanh(ax)$. We first characterize the~fixed-point structure of~$\sigma_a$ depending on~the~parameter $a>0$. We then analyze the~asymptotic behavior of~the~iterates $\sigma_a^{\,n}$, showing that they converge to~a~step-like limit as $n\to\infty$, see Figure~\ref{fig:sigma_n_step_function}. Finally, we quantify this transition by~deriving estimates on~the~convergence rate in~terms of~the~distance to~the~stable fixed points.

\begin{claim}[Properties of~$\sigma_a$]\label{cl:properties_of_sigma_and_sigma_n}
For all $a>0$ and~$n\in\N$, the~following properties hold:
\begin{enumerate}
    \item\label{it:range_of_g} $0 \le \sigma_a(x) < 1$ for~all $x \in [0,\infty)$,
    
    \item\label{it:oddness_of_g} $\sigma_a$ is odd, i.e.\ $\sigma_a(x)=-\sigma_a(-x)$,
    
    \item\label{it:monotonicity_of_g} $\sigma_a$ is strictly increasing on~$\mathbb R$, and
    \begin{equation}\label{eq:derivation_of_g}
    0 < \sigma_a'(x) = a\bigl(1-\sigma_a(x)^2\bigr)
    = \frac{a}{\cosh^2(ax)}
    \le a~\quad \forall x\in\mathbb R,
    \end{equation}
    and, for~$a>1$,
    \[
    \bar x = \frac{1}{a}\arccosh(\sqrt{a}),
    \]
    \item\label{it:concavity_of_g} $\sigma_a$ is strictly concave on~$\R^+$ and~strictly convex on~$\R^-$, and
    \begin{equation}\label{eq:second_der_of_g}
        \sigma_a''(x)=-2a\,\sigma_a(x)\,\sigma_a'(x)
        =-2a^2 \sigma_a(x)\bigl(1-\sigma_a(x)^2\bigr),
    \end{equation}
    
    \item\label{it:chain_rule_sigma_n'} for~every $n\in\N$ and~$x\in\mathbb R$, the~chain rule yields
    \[
        \big(\sigma_a^{\,n}(x) \big)'
        =\prod_{k=0}^{\,n-1} \sigma_a'\!\bigl(\sigma_a^{\,k}(x)\bigr).
    \]
\end{enumerate}
\end{claim}

\begin{proof}
All statements follow directly from~the~explicit form of~$\sigma_a$ and~standard calculus identities.
\end{proof}

\begin{lemma}[Fixed points of~$\sigma_a$]
If $0<a\le 1$, then $0$ is the~unique fixed point of~$\sigma_a$. For~$a>1$, the~function $\sigma_a$ has exactly two additional fixed points $\pm x_\star$ with $x_\star>0$. Moreover,
\[
x_\star>\bar x.
\]
\end{lemma}

\begin{proof}
Set
\[
h(x)=\sigma_a(x)-x.
\]
Then $x_\star$ is a~fixed point of~$\sigma_a$ if and~only if $h(x_\star)=0$. Clearly,
\[
h(0)=0
\]
and
\[
\lim_{x\to-\infty}h(x)
=
+\infty,
\qquad
\lim_{x\to\infty}h(x)
=
-\infty
\]
for every $a>0$. Moreover, by~\eqref{eq:derivation_of_g} and~\eqref{eq:second_der_of_g},
\[
h'(x)=a(1-\sigma_a(x)^2)-1,
\qquad
h''(x)=-2a^2\sigma_a(x)(1-\sigma_a(x)^2).
\]
For~$x>0$, we have
\[
0<\sigma_a(x)<1,
\]
and therefore
\[
h''(x)<0.
\]
Hence $h'$ is strictly decreasing on~$(0,\infty)$ for~every $a>0$ and, by~oddness, strictly increasing on~$(-\infty,0)$.

If $0<a<1$, then
\[
h'(x)\le a-1<0
\qquad \text{for all } x\in\R.
\]
If $a=1$, then
\[
h'(0)=0,
\qquad
h'(x)<0
\quad \text{for } x\neq0.
\]
Thus, for~all $0<a\le1$, the~function $h$ is strictly decreasing on~$\R$, and~therefore
\[
h(x)=0
\iff
x=0.
\]

Now let $a>1$. Then
\[
h'(0)=a-1>0,
\qquad
h'(\bar x)=0
\]
by the~definition of~$\bar x$. Since $h'$ is strictly decreasing on~$(0,\infty)$, it follows that
\[
h'(x)>0
\quad \text{for } x\in(0,\bar x),
\]
and
\[
h'(x)<0
\quad \text{for } x>\bar x.
\]
Consequently, $h$ is strictly increasing on~$[0,\bar x]$ and~strictly decreasing on~$[\bar x,\infty)$. In~particular,
\[
h(\bar x)>h(0)=0.
\]
Together with
\[
\lim_{x\to\infty}h(x)=-\infty,
\]
this implies that there exists a~unique zero
\[
x_\star>\bar x.
\]
By oddness of~$h$, also $-x_\star$ is a~zero. Hence the~fixed points of~$\sigma_a$ are exactly $0$ and~$\pm x_\star$.
\end{proof}

\begin{lemma}[Limit shape of~$\sigma_a^{\,n}$]\label{lemma:limit_of_gn_sgn}
    For~every $x\in\mathbb{R}$,
    \begin{equation*}
        \lim_{n\to\infty} \sigma_a^{\,n}(x) = \operatorname{sign}(x)\, x_\star,
    \end{equation*}
    where for~$0<a\le1$ we use the~convention $x_\star=0$.
\end{lemma}

\begin{proof}
    The~claim is trivial for~$0<a<1$, since then $\sigma_a$ is a~contraction on~$\mathbb R$ and~$x_\star=0$ is the~unique fixed point. Hence, by~Banach's fixed point theorem,
    \[
        \sigma_a^{\,n}(x)\to x_\star=0
        \qquad \text{for all } x\in\mathbb R.
    \]
    For~$a=1$, we have $|\sigma_a(x)|<|x|$ for~all $x\neq0$, which implies again that $\sigma_a^{\,n}(x)\to0=x_\star$.
    
    Now let $a>1$. Due to~monotonicity and~convexity of~$\sigma_a$, we have $x < \sigma_a(x) < x_\star$ for~$x \in (0, x_\star)$ and~$x_\star < \sigma_a(x) < x$ for~$x \in (x_\star, \infty)$. Hence,
    \[
        \sigma_a^{\,n-1}(x) < \sigma_a^{\,n}(x) < x_\star \quad \text{for } x \in (0, x_\star),
        \qquad
        x_\star < \sigma_a^{\,n}(x) < \sigma_a^{\,n-1}(x) \quad \text{for } x \in (x_\star, \infty).
    \]
    Thus, $(\sigma_a^{\,n}(x))_n$ is increasing and~bounded above by~$x_\star$ for~$x \in (0, x_\star)$, and~decreasing and~bounded below by~$x_\star$ for~$x \in (x_\star, \infty)$. Therefore, $\sigma_a^{\,n}(x)$ converges to~$x_\star$. The~case $x<0$ follows by~odd symmetry of~$\sigma_a$, and~$x=0$ is trivial.
\end{proof}

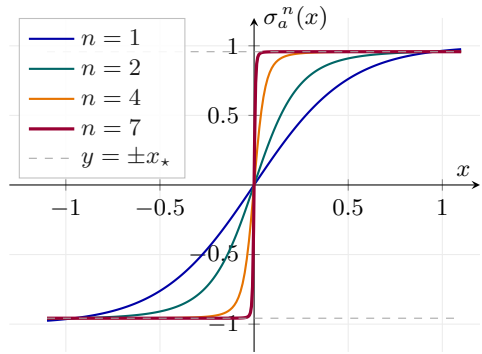
\begin{figure}[ht]
\centering
\begin{tikzpicture}
\begin{axis}[
    width=7.8cm, height=6cm,
    xlabel={$x$},
    ylabel={$\sigma_{a}^{\,n}(x)$},
    xmin=-1.3, xmax=1.2,
    ymin=-1.2, ymax=1.2,
    axis lines=center,
    xtick={-1,-0.5,0,0.5,1},
    ytick={-1,-0.5,0,0.5,1},
    ylabel style={at={(0.52,1)}, anchor=west},
    tick label style={font=\footnotesize},
    label style={font=\small},
    legend style={
        at={(0.02,1)}, anchor=north west,
        font=\footnotesize, draw=gray!50,
        fill=white, fill opacity=0.9
    },
    legend cell align=left,
    grid=both,
    grid style={gray!15},
    clip=true,
]

\addplot[blue!65!black, thick, domain=-1.1:1.1, samples=200]
    {tanh(2*x)};
\addlegendentry{$n=1$}

\addplot[teal!80!black, thick, domain=-1.1:1.1, samples=200]
    {tanh(2*tanh(2*x))};
\addlegendentry{$n=2$}

\addplot[orange!90!black, thick, domain=-1.1:1.1, samples=200]
    {tanh(2*tanh(2*tanh(2*tanh(2*x))))};
\addlegendentry{$n=4$}

\addplot[purple!80!black, very thick, domain=-1.1:1.1, samples=400]
    {tanh(2*tanh(2*tanh(2*tanh(2*tanh(2*tanh(2*tanh(2*x)))))))};
\addlegendentry{$n=7$}

\addplot[black!35, dashed, thin, domain=-1.1:1.1, forget plot] {0.9575};
\addplot[black!35, dashed, thin, domain=-1.1:1.1] {-0.9575};
\addlegendentry{$y=\pm x_\star$}

\end{axis}
\end{tikzpicture}
\caption{Iterates of~$\sigma_a^{\,n}$ for~$a=2$, converging to~the~step-like limit $\operatorname{sign}(x)\,x_\star$.}
\label{fig:sigma_n_step_function}
\end{figure}

\begin{claim}[Jump condition for~$\tau$]\label{claim:tau_choosing}
Let $a>1$ be fixed. Since $\sigma_a(x)>x$ for~$x\in(0,x_\star)$, there exists a~constant
\[
\tau_{\max}\in \bar x \cdot
\left[
1 - \min\left\{
2 \rho(a),\;
\frac{1}{a \bar x}\cdot\operatorname{arctanh}\!\left(\frac{2a\bar x}{a+1}\right)
\right\},
\;
1-2 \chi(a)
\right]
\]
such that for~all $0<\tau\le\tau_{\max}$,
\begin{equation}\label{eq:jump}
\sigma_a(\bar x-\tau)>\bar x+\tau.
\end{equation}
By the~symmetry of~$\sigma_a$, it also follows that
$\sigma_a'(-\bar{x})=1$ and~$\sigma_a(-\bar{x}+\tau)<-\bar{x}-\tau$ for~all
$0<\tau\le\tau_{\max}$.
\end{claim}

\begin{proof}
\textbf{Upper bound.}
Since $\sigma_a(x)\le ax$ for~all $x\ge 0$, a~necessary condition for~\eqref{eq:jump} is
\[
\bar x + \tau < \sigma_a(\bar x-\tau) \le a(\bar x-\tau),
\]
which yields
\begin{equation}\label{eq:upper_bound_on_tau}
\tau < \frac{a-1}{a+1}\,\bar x.
\end{equation}

\textbf{Lower bound.}
From~\cite[Lemma~4]{bagul2022tight}, for~all $x>0$,
\begin{equation}\label{eq:bounds_on_tanh}
\frac{x}{\sqrt{1+x^2}} \le \tanh(x).
\end{equation}
Hence a~sufficient condition for~\eqref{eq:jump} is
\[
\sigma_a(\bar x-\tau)
\ge \frac{a(\bar x-\tau)}{\sqrt{1+(a(\bar x-\tau))^2}}
\ge \bar x+\tau.
\]
This holds if
\[
a^2(\bar x-\tau)^2 \bigl(1-(\bar x+\tau)^2\bigr) \ge (\bar x+\tau)^2.
\]
Using \eqref{eq:upper_bound_on_tau}, it suffices to~impose the~stronger condition
\[
a^2(\bar x-\tau)^2\left(1-\left(\frac{2a\bar x}{a+1}\right)^2\right)\ge (\bar x+\tau)^2.
\]
Set
\[
A:=a^2-4\frac{a^4\bar x^2}{(a+1)^2}.
\]
By Lemma~\ref{lemma:arccosh_upper_bound},
\[
\arccosh(\sqrt a)\le \sqrt{a-1},
\]
and hence
\[
A
\ge
\frac{a^2}{(a+1)^2}(a^2-2a+5)
>0.
\]
The previous inequality then reduces to
\[
\tau^2(A-1)-2\bar x\tau(A+1)+\bar x^2(A-1)\ge0.
\]
Its discriminant is
\[
D = 16 \bar x^2 A,
\]
with roots
\[
\tau_{\pm}
=
\bar x\left(1 + 2\frac{1 \pm \sqrt{A}}{A-1}\right).
\]
Since $\tau_+ > \bar x$ is not admissible, we obtain $\tau_{\max} \ge \tau_-$. Using
\[
\frac{1-\sqrt{A}}{A-1} = -\frac{1}{\sqrt{A}+1},
\]
yields the~claimed lower bound.

\medskip
\emph{An alternative lower bound} can be derived as follows. Using \eqref{eq:upper_bound_on_tau}, it suffices to~require
\[
\sigma_a(\bar x-\tau)\ge \frac{2a\bar x}{a+1},
\]
which implies
\[
\tau_{\max} \ge \bar x - \frac{1}{a}\operatorname{arctanh}\!\left(\frac{2a\bar x}{a+1}\right).
\]
\end{proof}

\begin{lemma}[Bounds on~$\arccosh$]\label{lemma:arccosh_upper_bound}
For all $y\ge1$, we have
\[
\arccosh(y) \le \sqrt{2(y-1)},
\]
and, more crudely,
\[
\arccosh(y) \le \sqrt{y^2-1}.
\]
In particular,
\begin{equation}\label{eq:numerical_bound_arccosh/y+1}
\frac{\arccosh(\sqrt{y})}{y+1}
\le \frac{31}{100}.
\end{equation}
\end{lemma}

\begin{proof}
By~definition,
\[
\arccosh(y) = \ln\bigl(y+\sqrt{y^2-1}\bigr).
\]
Define
\[
\gamma(y) = \sqrt{2(y-1)} - \arccosh(y).
\]
A direct computation shows
\[
\gamma'(y)
=
\frac{\sqrt{y+1}-\sqrt{2}}{\sqrt{2(y^2-1)}}
> 0
\qquad \text{for } y > 1.
\]
Thus $\gamma$ is increasing and~$\gamma(y)\ge \gamma(0)=0$, which yields
\[
\arccosh(y) \le \sqrt{2(y-1)}.
\]

For the~second bound, define
\[
\delta(y) = \sqrt{y^2-1} - \arccosh(y).
\]
Then
\[
\delta'(y)
=
\frac{y-1}{\sqrt{y^2-1}}
> 0
\qquad \text{for } y> 1.
\]
Hence $\delta$ is increasing with $\delta(y)\ge\delta(1)=0$, which proves
\[
\arccosh(y) \le \sqrt{y^2-1}.
\]

Applying the first $\arccosh$ bound with argument $\sqrt{y}$ yields
\begin{equation}\label{eq:sharp_upper_bound_arccosh(sqrt(a))}
    \arccosh(\sqrt{y})\leq\sqrt{2(\sqrt{y}-1)}.
\end{equation}
Hence,
\[
\frac{\arccosh(\sqrt{y})}{y+1}
\leq
\frac{\sqrt{2(\sqrt{y}-1)}}{y+1}.
\]
The function $\frac{\sqrt{2(\sqrt{y}-1)}}{y+1}$ attains its maximum at~$y_0 = \frac{11 + 4\sqrt{7}}{9}$ and~this maximum is bounded by~$\frac{31}{100}$.
\end{proof}

\begin{corollary}\label{col:lower_bound_xi}
For $a>1$, it holds that
\[
\frac{a}{\cosh^2\!\left(\arccosh(\sqrt{a}) \cdot 2\chi(a)\right)}
\;\geq\;
\sigma_a'(\bar x - \tau_{\max})
\;\geq\;
\frac{a}{\cosh^2\!\left(\arccosh(\sqrt{a}) \cdot 2\rho(a)\right)} > 1,
\]
and
\[
\frac{a}{\cosh^2\!\left(\arccosh(\sqrt{a}) \cdot 2(1-\chi(a))\right)}
\;\leq\;
\sigma_a'(\bar x + \tau_{\max})
\;\leq\;
\frac{a}{\cosh^2\!\left(\arccosh(\sqrt{a}) \cdot 2(1-\rho(a))\right)} < 1.
\]
\end{corollary}

\begin{proof}
From Claim~\ref{claim:tau_choosing} we obtain the~bounds
\[
2\bar x\chi(a) \leq \bar x - \tau_{\max} \le 2\bar x\rho(a)
\]
and
\[
2\bar x \!\left(1-\rho(a)\right)
\ge
\bar x + \tau_{\max}
\ge
2\bar x\!\left(1 - \chi(a)\right).
\]
Since $\operatorname{arccosh}(\sqrt a)=\ln(\sqrt a+\sqrt{a-1})$, substituting the~above bounds on~$\bar x \pm \tau_{\max}$ into \eqref{eq:derivation_of_g}
yields the~stated inequalities. In~particular,
$\sigma_a'(\bar x-\tau_{\max})>1$ and~$\sigma_a'(\bar x+\tau_{\max})<1$.
\end{proof}

\begin{corollary}\label{col:lower_bound_xi_simplier}
For $a>1$, the~following simpler bounds hold:
\[
\sigma_a'(\bar x + \tau_{\max})
\le
\frac{a}{\pi(a)\cdot\left[\frac{1+\pi(a)^{-1}}{2}\right]^2}
=
O(a^{-1}),
\qquad
\sigma_a'(\sigma_a(\bar x + \tau))
\le
\frac{a}{\cosh^2 \left( a\cdot \frac{\pi(a)-1}{\pi(a)+1}\right)}
=
O\left( \frac{a}{e^{2a}}\right).
\]

\end{corollary}

\begin{proof}
The first estimate follows directly from
\[
e^{\arccosh(\sqrt{a})} = \sqrt{a} + \sqrt{a-1}.
\]
For~the~second estimate, we use the~identity
\[
\tanh(y) = 1 - \frac{2}{e^{2y} + 1}
\]
and obtain
\[
\sigma_a(\bar x + \tau_{\max}) 
\ge
1-\frac{2}{\big( \sqrt{a}+\sqrt{a-1} \big) ^{4(1-\rho(a))}+1}.
\]
Combining this with \eqref{eq:derivation_of_g} yields the~final bound.
\end{proof}

\begin{lemma}[Exit time bounds]\label{lemma:n_0(x)}
    Let $a>1$. For~all $x\in I$, it holds that
    \begin{equation}\label{eq:bounds_on_n_0(x)}
       \frac{\ln\left(\frac{\bar{x}-\tau}{|x|}\right)}{\ln a} \leq n_0(x) \leq \frac{\ln\left(\frac{\bar{x}-\tau}{|x|}\right)}{\ln \xi}.
    \end{equation}
\end{lemma}

\begin{proof}
    Applying the~mean value theorem with the~fact $\sigma_a^{\,n}(0)=0$ and~by~the~chain rule Claim~\ref{cl:properties_of_sigma_and_sigma_n}, item~\ref{it:chain_rule_sigma_n'}, we obtain
    \[
        \bar{x} - \tau > \sigma_a^{\,n}(x) \ge \inf_{y \in (0,\, \bar{x} - \tau)} \big(\sigma_a^{\,n}(y)\big)' \cdot x 
        \ge \dots \ge \xi^n x.
    \]

    For~the~lower bound, the~necessary condition is
    \[\bar x-\tau > a^{n} \, x \ge  (\sup_{y \in (0,\, \bar{x} - \tau)} \sigma_a'(y))^{n}\cdot x \ge \sigma_a^{\,n}(x),\] which implies the~result.

    For~$x<0$, the~argument is analogous and~follows by oddness of $\sigma_a^{\,n}$.
\end{proof}

\begin{lemma}[Attraction estimates]\label{lemma:basin_of_attraction}
For every $a>1$, the following hold:
\begin{enumerate}
\item Let $x\in I$ and~$k\in\mathbb{N}$ with $k\ge3$. Then
\[
x_\star-\operatorname{sign}(x)\,\sigma_a^{\,n_0(x)+k}(x)
\le
\eta\,\bar\eta^{k-3}(x_\star-\bar x-\tau).
\]

\item For every $n\in\mathbb{N}$ and
$\bar x+\tau\le |x|\le x_\star$,
\[
\mu^n(x_\star-\operatorname{sign}(x)x)
\le
x_\star-\operatorname{sign}(x)\sigma_a^{\,n}(x)
\le
\sigma_a'(x)\,(\sigma_a'(\sigma_a(x)))^{n-1}
(x_\star-\operatorname{sign}(x)x).
\]

\item For every $n\in\mathbb{N}$ and~$|x|\ge x_\star$,
\[
(\sigma_a'(x))^n
\left(\operatorname{sign}(x)x - x_\star \right)
\le
\operatorname{sign}(x)\sigma_a^{\,n}(x) - x_\star
\le
\mu^{n}(\operatorname{sign}(x)x-x_\star).
\]
\end{enumerate}
\end{lemma}

\begin{proof}
We prove the three statements and~restrict to $x>0$ by oddness.

\textbf{Key estimate.}
For any $n\ge1$ and~$z$,
\[
(\sigma_a^{\,n}(z))'
=
\prod_{j=0}^{n-1}\sigma_a'(\sigma_a^{\,j}(z)).
\]
Using monotonicity of $\sigma_a'$ along forward orbits, we obtain
\[
(\sigma_a^{\,n}(z))'
\le
\sigma_a'(x)\,(\sigma_a'(\sigma_a(x)))^{n-1}
\quad\text{if } z\ge x,
\]
and
\[
(\sigma_a^{\,n}(z))'
\ge
(\sigma_a'(x))^n
\quad\text{if } z\le x.
\]

\medskip
\textbf{1.}
Let $x\in(0,\bar x-\tau)$. By definition of the first exit time,
\[
\bar x-\tau 
\le
\sigma_a^{\,n_0(x)+1}(x)
,
\]
hence by \eqref{eq:jump},
\[
\bar x+\tau < \sigma_a^{\,n_0(x)+2}(x) < x_\star.
\]
Applying the mean value theorem to $\sigma_a^{k-2}$ on the interval between
$\sigma_a^{\,n_0(x)+2}(x)$ and~$x_\star$, we obtain
\[
x_\star-\sigma_a^{\,n_0(x)+k}(x)
=
(\sigma_a^{\,k-2}(z))'
\bigl(x_\star-\sigma_a^{\,n_0(x)+2}(x)\bigr),
\]
for some $z$ in that interval. The key estimate then yields
\[
x_\star-\sigma_a^{\,n_0(x)+k}(x)
\le
\eta\,\bar\eta^{k-3}(x_\star-\bar x-\tau).
\]

\medskip
\textbf{2.}
Fix $\bar x+\tau\le x\le x_\star$. Apply the mean value theorem to $\sigma_a^n$ between $x$ and~$x_\star$ and~use the key estimate.

\medskip
\textbf{3.} Let $|x|\ge x_\star$. The same argument as in part 2 applies, using the reversed monotonicity direction, which yields the bounds with $(\sigma_a'(x))^n$ and~$\mu^n$.
\end{proof}

\begin{lemma}[Basin of~attraction]\label{lemma:error_of_basin_of_attraction}
Let $n \in \mathbb{N}$ and~$\kappa \in (0,1)$.
Then for~all $x\in\R$ with $|x|\ge \xi_n$, it holds that
\begin{equation}\label{eq:error_of_basin_of_attraction}
|x_\star - \operatorname{sign}(x)\,\sigma_a^{\,n}(x)|
\;\le\; \e_n.
\end{equation}
\end{lemma}

\begin{proof}
Define
\[
x(\kappa n) \;=\; \sup\Big\{ |x| : x \in (-x_\star, x_\star)\ \wedge\ n_0(x) \ge \lceil \kappa n \rceil\Big\}.
\] 
By Lemma~\ref{lemma:n_0(x)}, any $x$ with $n_0(x)\ge \lceil \kappa n \rceil$ satisfies
\[
\lceil \kappa n \rceil \ln\xi \le \ln\!\Big(\frac{\bar x-\tau}{|x|}\Big),
\]
hence $|x|\le \xi^{-\lceil  \kappa n \rceil }(\bar x-\tau)=:\xi_n$. Thus $x(\kappa n)\le\xi_n$.

Now take $x\in I$ with $|x|\ge x(\kappa n)$. Then $n_0(x)\le \lceil \kappa n \rceil$, and~setting
$k:=n-n_0(x)\ge \lfloor (1-\kappa)n \rfloor$, Lemma~\ref{lemma:basin_of_attraction} yields
\[
x_\star - \operatorname{sign}(x)\sigma_a^{\,n}(x)
\le
\eta\,\bar\eta^{\,k-3}(x_\star-\bar x-\tau)
\le
\eta\,\bar\eta^{\,\lfloor(1-\kappa)n\rfloor-3}(x_\star-\bar x-\tau)
<
\e_n.
\]
Since $x(\kappa n)\le \xi_n$, the same bound holds for~all $x\in I$ with $|x|\ge \xi_n$.

It remains to extend the estimate outside $I$. We set $n_0(x)=0$ for~$x\in\mathbb{R}\setminus I$. On $(0,\infty)$ the function $\sigma_a'$ is decreasing, and~since $\bar x+\tau<x_\star$ by \eqref{eq:jump}, we have
\[
\mu:=\sigma_a'(x_\star) < \eta,\bar\eta.
\]
Hence,
\[
\mu^n \le \eta\,\bar\eta^{\,n-3},
\qquad
\sigma_a'(x)\,(\sigma_a'(\sigma_a(x)))^{n-1}
\le \eta\,\bar\eta^{\,n-3}
\quad \text{for } \bar x+\tau\le x\le x_\star.
\]
Applying Lemma~\ref{lemma:basin_of_attraction}, parts 2 and~3, completes the extension to $x\in\mathbb{R}\setminus I$.
\end{proof}

 \subsection{Analysis of~the~coarse bump}\label{subsection:bump_analysis}

In this subsection we analyze the~coarse bump constructed in~the~second layer. We study its behavior in~higher dimensions and~derive conditions ensuring sufficient decay outside a~prescribed cube. In~particular, we show that, for~a~suitable choice of~parameters, the~bump can be made arbitrarily small outside $[-y,y]^d$. The~result relies on~several auxiliary lemmas, which are proved subsequently.

\begin{lemma}[Coarse bump separation]\label{lemma:condition_on_alpha_and_M}
Let $q,d,B\in\N$ and~$c>0$. For~simplicity, denote $A = B^{-1/q}(c-\tanh(\frac c2))$ and~$C = B^{-1/q}\tanh(\frac c2)$ with $K:=\cosh(2C)$. Let
\begin{equation}\label{eq:assumption_xi_L}
    \tilde \xi_n \le \frac{\tanh(C)}{d}\cdot\frac{\cosh(2A)-1}{\cosh(2A)+K}
\end{equation}
and
\begin{equation}\label{eq:assumption_on_alpha}
    0
    <
\alpha
\leq
\frac{1}{2A}\arccosh\!\left(1+\frac{1+K}{2d}\cdot\frac{\cosh(2A)-1}{\cosh(2A)+K}\right)\,.
\end{equation}
Then for~all
\[
\sqrt{\frac{10(1+\cosh(2C))}{9\tanh(C)A^2}\,d \,\tilde\xi_n}\le y<1,
\]
it holds that
\begin{equation}\label{eq:cube_condition}
\frac{d-1}{d}f(0)+\frac{1}{d}f(y)-f(\alpha y)\le -\tilde\xi_n .
\end{equation}
\end{lemma}

\begin{proof}
Using
\[
\tanh(y)=\frac{\sinh(y)}{\cosh(y)}
\]
together with the~identities
\[
\sinh(u)+\sinh(v)=2\sinh\!\left(\frac{u+v}{2}\right)\cosh\!\left(\frac{u-v}{2}\right)
\]
and
\[
\cosh(u)+\cosh(v)=2\cosh\!\left(\frac{u+v}{2}\right)\cosh\!\left(\frac{u-v}{2}\right),
\]
and using that $\sinh$ is odd, we obtain
\begin{equation}\label{eq:f(y)}
f(y)
= \frac{\sinh(Ay+C)}{\cosh(Ay+C)} - \frac{\sinh(Ay-C)}{\cosh(Ay-C)}
= \frac{2\sinh(2C)}{\cosh(2Ay)+\cosh(2C)}.
\end{equation}
Then, condition \eqref{eq:cube_condition} is equivalent to
\begin{equation}\label{eq:red}
\frac{d-1}{d}\frac{1}{1+K}
+
\frac1d\frac{1}{\cosh(2Ay)+K}
-
\frac{1}{\cosh(2A\alpha y)+K}
\le
-\frac{\tilde\xi_n}{2\sinh(2C)}.
\end{equation}

\medskip
\noindent
\textbf{Step 1: Verification of~sharp inequality of~\eqref{eq:red} at~$y=1$.}

By definition of~$\alpha$,
\[
\cosh(2A\alpha)
\leq
1+\frac{1+K}{2d}\cdot
\frac{\cosh(2A)-1}{\cosh(2A)+K}.
\]
Since $\tilde\xi_n \le \frac{\tanh(C)}{2d}\frac{\cosh(2A)-1}{\cosh(2A)+K}$ and
\begin{equation}\label{eq:tanh/sinh}
    \frac{\tanh(C)}{\sinh(2C)}=\frac{1}{1+K},
\end{equation}
we obtain
\[
-\frac{\tilde\xi_n}{2\sinh(2C)}
\ge
-\frac{1}{2d}\frac{1}{1+K}
\frac{\cosh(2A)-1}{\cosh(2A)+K}.
\]
Hence it suffices to~verify, for~$y=1$, that
\begin{equation}\label{eq:goal_y1}
\frac{1}{\cosh(2A\alpha)+K}
>
\frac{d-1}{d}\frac{1}{1+K}
+
\frac1d\frac{1}{\cosh(2A)+K}
+
\frac{1}{2d}\frac{1}{1+K}
\frac{\cosh(2A)-1}{\cosh(2A)+K}.
\end{equation}

For the left-hand side of~\eqref{eq:goal_y1}, we have
\[
\frac{1}{\cosh(2A\alpha)+K}
\geq
\frac{2d(\cosh(2A)+K)}
{(1+K)\bigl(2d(\cosh(2A)+K)+\cosh(2A)-1\bigr)}.
\]
The right-hand side of~\eqref{eq:goal_y1} equals
\[
\frac{2d(\cosh(2A)+K)-\cosh(2A)+1}
{2d(1+K)(\cosh(2A)+K)}.
\]
Therefore, to prove \eqref{eq:goal_y1}, it suffices to verify that
\[
\frac{2d(\cosh(2A)+K)}
{2d(\cosh(2A)+K)+\cosh(2A)-1}
\ge
\frac{2d(\cosh(2A)+K)-\cosh(2A)+1}
{2d(\cosh(2A)+K)}.
\]
Multiplying by the positive denominators, this inequality is equivalent to
\[
(\cosh(2A)-1)^2 \ge 0,
\]
which clearly holds. This proves
\eqref{eq:goal_y1} and~therefore establishes the sharp inequality in
\eqref{eq:red} for~$y=1$.

\medskip
\noindent
\textbf{Step 2: Lower bound on~$y$.}

Define
\[
G(y):=\frac{d-1}{d}f(0)+ \tilde\xi_n+\frac{1}{d}f(y)-f(\alpha y)
= \frac{d-1}{d}f(0)+ \tilde\xi_n + F(y),
\]
where $F$ is the~same function as~in~Lemma~\ref{lemma:extrema_of_F}. Therefore, $G$ is continuous on~$(0,\infty)$. From Step 1 we know that $G(1)<0$, while
\[
G(0) = \tilde\xi_n \quad \text{and} \quad \lim_{y\to\infty}G(y)=\frac{d-1}{d}f(0)+\tilde\xi_n>0.
\]
Since $G(0)>0$ and~$G(1)<0$, there exists $y_1\in(0,1)$ with $G(y_1)=0$. Moreover, since $G(1)<0$ and~$\lim_{y\to\infty}G(y)>0$, there exists also $y_2>1$ such that $G(y_2)=0$. Since $G'(y)=F'(y)$, both functions share the~same critical points. By~Lemma~\ref{lemma:extrema_of_F}, the~function $F$ is either increasing, or~it has a~unique critical point $y_0>0$, and~the~latter occurs if and~only if $\alpha^2 d<1$.

Since $G$ has at~least two distinct zeros, the~monotone case is impossible. Hence $\alpha^2 d<1$, and~$F$ (and $G$) has a~unique global minimum at~some $y_0>0$, so $G$ is strictly decreasing on~$(0,y_0]$ and~strictly increasing on~$[y_0,\infty)$. This implies the~existence of~exactly two points $y_1<y_0<y_2$ such that
\[
G(y_1)=G(y_2)=0.
\]

It remains to~derive an~explicit upper bound on~$y_1$.

\medskip

To do so, we use quadratic bounds for~the~hyperbolic cosine. Since $2A\alpha$ is controlled by~assumption, the~argument of~$\cosh(2A\alpha y)$ remains in~a~regime where a~quadratic estimate is sufficiently accurate. Observe that
\[
\frac{\cosh(t)-1}{t^2}
=
\frac12+\frac{t^2}{4!}+\frac{t^4}{6!}+\cdots
\]
is increasing on~$\R^+$. Hence, for~$0<y<1$,
\[
\frac{\cosh(2A\alpha y)-1}{(2A\alpha y)^2}
\le
\frac{\cosh(2A\alpha)-1}{(2A\alpha)^2},
\]
which implies
\[
\cosh(2A\alpha y)-1
\le
y^2\bigl(\cosh(2A\alpha)-1\bigr).
\]
Using the~definition of~$\alpha$,
\[
\cosh(2A\alpha)-1
\le
\frac{1+K}{2d}
\frac{\cosh(2A)-1}{\cosh(2A)+K},
\]
and therefore
\begin{equation}\label{eq:quadratic_bound_cosh}
\cosh(2A\alpha y)
\le
1+\beta y^2,
\end{equation}
where
\[
\beta
:=
\frac{1+K}{2d}
\frac{\cosh(2A)-1}{\cosh(2A)+K}.
\]

On the~other hand, since all terms in~the~Taylor expansion of~$\cosh$ are nonnegative, as can be seen above, we get for~$y<1$,
\[
\cosh(2Ay)
\ge
1+2A^2y^2.
\]
Using the~elementary bound
\[
\frac{1}{1+u}\ge 1-u
\quad \text{for } u\ge0,
\]
we obtain
\[
\frac{1}{\cosh(2A\alpha y)+K}
\ge
\frac{1}{1+K+\beta y^2}
\ge
\frac{1}{1+K}\left(1-\frac{\beta}{1+K}y^2\right).
\]

Hence a~sufficient condition for~\eqref{eq:red} is
\[
\frac{1}{1+K}\left(1-\frac{\beta}{1+K}y^2\right)
-
\frac1d\frac{1}{1+K+2A^2y^2}
\ge
\frac{\tilde\xi_n}{2\sinh(2C)}
+
\frac{d-1}{d}\frac{1}{1+K}.
\]
After clearing denominators this reduces to
\[0\geq \frac{2A^2\beta}{(1+K)^2}y^4 + \left( \frac{\beta}{(1+K)} + \frac{\tilde\xi_n}{2\sinh(2C)}2A^2 - \frac{1}{d(1+K)}2A^2 \right) y^2 + \frac{1+K}{2\sinh(2C)}\tilde\xi_n,\]
which can be denoted as
\[
0\ge
\eta\,(y^2)^2+\mu\,(y^2)+\nu,
\]
with $\eta,\nu>0$. Hence necessarily $\mu<0$. Using the~assumption on~$\tilde\xi_L$ with \eqref{eq:tanh/sinh} one obtains
\[
-\mu
\ge
\frac{1}{d(1+K)}
\left(
2A^2
-
\frac12
\frac{\cosh(2A)-1}{\cosh(2A)+K}(A^2+1+K)
\right).
\]
By Lemma~\ref{lemma:bound_on_-B} for~$x:=A$ this yields
\[
-\mu\ge \frac{9A^2}{10d(1+K)}.
\]

The roots satisfy
\[
y_\pm^2=\frac{-\mu\pm\sqrt{\mu^2-4\eta\nu}}{2\eta}.
\]
Since we seek the~smallest admissible $y$, we consider $y_-$. Using
\[
-\mu\ge\sqrt{\mu^2-4\eta\nu}\ge0,
\]
we obtain
\[
\frac{\nu}{-\mu}\le y_-^2
=
\frac{4\eta\nu}{2\eta(-\mu+\sqrt{\mu^2-4\eta\nu})}
\le
\frac{2\nu}{-\mu}.
\]
Using the~upper bound on~the~root gives the~sufficient lower bound
\[
\sqrt{\frac{10(1+K)^2}{9\sinh(2C)}\cdot
\frac{d\,\tilde\xi_n}{A^2}}
\le y<1.
\]
Finally, application of \eqref{eq:tanh/sinh} completes the~proof.
\end{proof}

\begin{lemma}\label{lemma:extrema_of_F}
Let $q,d,B\in\N$, $0<\alpha<1$ and~$c>0$. For~simplicity, denote $A = B^{-1/q}(c-\tanh(\frac c2))$ and~$C = B^{-1/q}\tanh(\frac c2)$ with $K:=\cosh(2C)$. Define
\[
F(y)=\frac1d f(y)-f(\alpha y).
\]
Then:
\begin{enumerate}
    \item $F(0) = -\frac{d-1}{d} f(0) \le 0$, $\lim_{y\to\infty}F(y)=0$ and~$F$ is continuous on~$\R^+$.
    \item If $\alpha^2 d<1$, the~function $F$ has a~unique critical point $y_0>0$, which is a~global minimum; otherwise, $F$ is increasing on~$(0,\infty)$. In~particular, for~$d=1$ the~former case always occurs.
\end{enumerate}

\end{lemma}

\begin{proof}
Since $f(0)=2\tanh(C)$, we have
\[
F(0)=\left(\tfrac1d-1\right)f(0) \le 0,
\qquad
\lim_{y\to\infty}f(y)=0,
\]
hence $\lim_{y\to\infty}F(y)=0$.

Define
\[
\lambda(y)
=
\frac{\sinh(2Ay)}{\left(\cosh(2Ay)+\cosh(2C)\right)^2}.
\]
Using \eqref{eq:f(y)}, direct computation yields
\begin{equation}\label{eq:der_of_f}
    f'(y)=-4A\sinh(2C)\,\lambda(y),
\end{equation}
and thus
\[
F'(y)
=
\frac{4A\sinh(2C)\,\lambda(\alpha y)}{d}
\left(
\alpha d - \frac{\lambda(y)}{\lambda(\alpha y)}
\right).
\]
Hence,
\[
F'(y)\gtrless 0
\quad\Longleftrightarrow\quad
R(y):=\frac{\lambda(y)}{\lambda(\alpha y)} \lessgtr \alpha d.
\]

We next study the~monotonicity of~$R$. Differentiating yields
\[
(\ln R(y))'
=
\frac{\lambda'(y)}{\lambda(y)}
-
\alpha
\frac{\lambda'(\alpha y)}
     {\lambda(\alpha y)}.
\]
Since $2C\le 2$, set $\upsilon := 2C \in (0,2)$. Then $\lambda$ is of~the~form in~Lemma~\ref{lemma:x*omega(x)_decreasing}, so
\[
x\mapsto x\frac{\lambda'(x)}{\lambda(x)}
\]
is decreasing. Since $\alpha y<y$, we have
\[
\alpha y
\frac{\lambda'(\alpha y)}
     {\lambda(\alpha y)}
>
y
\frac{\lambda'(y)}
     {\lambda(y)}.
\]
Dividing by~$y>0$ gives
\[
(\ln R(y))'<0.
\]
Hence $R$ is strictly decreasing on~$(0,\infty)$.

Moreover,
\[
\lim_{y\to 0^+}R(y)=\frac{1}{\alpha},
\qquad
\lim_{y\to\infty}R(y)=0.
\]
Now we distinguish cases.

If $\alpha^2 d \ge 1$, then $\alpha d \ge \frac1\alpha$, hence
\[
R(y)\le \frac1\alpha \le \alpha d \quad \text{for all } y>0,
\]
so $F'(y)\ge 0$ and~$F$ is increasing on~$(0,\infty)$.

If $\alpha^2 d < 1$, then $\alpha d < \frac1\alpha$. Since $R$ is continuous, strictly decreasing, and
\[
R(0^+)=\frac1\alpha > \alpha d > 0 = \lim_{y\to\infty}R(y),
\]
there exists a~unique $y_0>0$ such that
\[
R(y_0)=\alpha d.
\]
Equivalently $F'(y_0)=0$, and~monotonicity of~$R$ implies
\[
F'(y)<0 \text{ for~} y<y_0,
\qquad
F'(y)>0 \text{ for~} y>y_0.
\]
Thus $y_0$ is a~unique global minimum.
\end{proof}

\begin{lemma}\label{lemma:x*omega(x)_decreasing}
Let $0<\upsilon\le2$ and~define
\[
\lambda(x)
=
\frac{\sinh(x)}
{\bigl(\cosh(x)+\cosh(\upsilon)\bigr)^2},
\qquad x>0.
\]
Then the~mapping
\[
x\mapsto x\,\frac{\lambda'(x)}{\lambda(x)}
\]
is decreasing on~$\R^+$.
\end{lemma}

\begin{proof}
Set
\[
C_x=\cosh(x),\qquad
S_x=\sinh(x),\qquad
\Upsilon=\cosh(\upsilon).
\]
Since $0<\upsilon\le2$,
\[
1\le \Upsilon\le \cosh(2)<4.
\]

Denote $\Lambda(x) := \frac{\lambda'(x)}{\lambda(x)}$. A~direct computation yields
\[
\Lambda(x)
=
\frac{C_x}{S_x}
-
2\frac{S_x}{C_x+\Upsilon},
\qquad
\Lambda'(x)
=
-\frac1{S_x^2}
-
2\frac{1+C_x\Upsilon}{(C_x+\Upsilon)^2}.
\]
To prove the~claim, it suffices to~show
\[
(x\Lambda(x))'
=
\Lambda(x)+x\Lambda'(x)
<0.
\]
Multiplying by~the~positive quantity $S_x^2(C_x+\Upsilon)^2$ and~using
\eqref{eq:cosh^2-sinh^2=1}, this is equivalent to
\[
F(x,\Upsilon)<0,
\]
where
\begin{equation}\label{eq:F}
F(x,\Upsilon)
=
-C_x^3S_x
-2xC_x^3\Upsilon
+C_xS_x(2+\Upsilon^2)
-3xC_x^2
+2S_x\Upsilon
+x(2-\Upsilon^2).
\end{equation}

We next study the~dependence on~$\Upsilon$. For~fixed $x$, differentiating twice yields
\[
\partial_\Upsilon^2F(x,\Upsilon)
=
2(C_xS_x-x).
\]
Since $S_x>x$ and~$C_x\ge1$, we have $C_xS_x>x$, hence
\[
\partial_\Upsilon^2F(x,\Upsilon)>0.
\]
Thus, for~each fixed $x>0$, the~function $\Upsilon\mapsto F(x,\Upsilon)$
is convex on~$[1,4]$. Therefore,
\[
F(x,\Upsilon)
\le
\max\{F(x,1),F(x,4)\},
\]
and it suffices to~show
\[
F(x,1)<0,
\qquad
F(x,4)<0.
\]

For $\Upsilon=1$, \eqref{eq:F} becomes
\[
F(x,1)
=
-C_x^3S_x
-2xC_x^3
+3C_xS_x
-3xC_x^2
+2S_x+x.
\]
Using
\begin{equation}\label{eq:sinh<x*cosh}
    S_x<xC_x
\end{equation}
(equivalently $\tanh(x)<x$), we obtain
\[
-C_x^3S_x+x<0,
\qquad
-2xC_x^3+2S_x<0,
\qquad
3C_xS_x-3xC_x^2<0.
\]
Summing these inequalities yields $F(x,1)<0$.

For $\Upsilon=4$,
\[
F(x,4)
=
-C_x^3S_x
-8xC_x^3
+18C_xS_x
-3xC_x^2
+8S_x
-14x.
\]
Differentiation with respect to~$x$ and~simplification using \eqref{eq:cosh^2-sinh^2=1} with the bound \eqref{eq:sinh<x*cosh} give
\begin{align*}
\partial_x F(x,4)
    & = -3C_x^2S_x^2 - C_x^4 - 8C_x^3 - 24xC_x^2S_x + 18S_x^2+18C_x^2 - 3C_x^2 - 6xC_xS_x + 8C_x-14 \\
& <
-4C_x^4- 32C_x^3 + 30C_x^2 + 32C_x - 26 = 
-2(C_x^2-1)(2C_x^2+16C_x-13).
\end{align*}
Since $C_x\ge1$, we obtain
\[
\partial_x F(x,4)<0.
\]
Moreover $F(0,4)=0$, hence
\[
F(x,4)<0
\qquad \text{for } x>0.
\]

Therefore $F(x,\Upsilon)<0$ for~all admissible $\Upsilon$, implying
\[
(x\Lambda(x))'<0.
\]
Hence $x\mapsto x\Lambda(x)$ is decreasing on~$\R^+$.
\end{proof}

\begin{claim}\label{claim:bounds_on_cosh/cosh}
    Let $0<N(x) \leq x$ for~$x>0$. Then for~all $x>0$,
    \begin{equation}\label{eq:cosh/cosh<tanh^2}
        \frac{1}{1+\tanh^2(N(x))} \tanh^2(x) \leq\frac{\cosh(2x)-1}{\cosh(2x)+\cosh(2N(x))} \leq \tanh^2(x).
    \end{equation}
\end{claim}

\begin{proof}
    Define
\[
r(x):=\frac{\cosh(2N(x))-1}{\cosh^2(x)}>0 .
\]
Using
\begin{equation}\label{eq:identities_cosh^2_sinh^2}
    2\cosh^2(x)=\cosh(2x)+1,\qquad
2\sinh^2(x)=\cosh(2x)-1,
\end{equation}
we obtain
\begin{equation*}
    \frac{\cosh(2x)-1}{\cosh(2x)+\cosh(2N(x))}
=
\frac{2\sinh^2(x)}{2\cosh^2(x)+\cosh(2N(x))-1}
=
\tanh^2(x)\frac{2}{2+r(x)}
\le \tanh^2(x).
\end{equation*}

For the~lower bound, note that $\cosh(2N(x))-1=2\sinh^2(N(x))$, hence
\[
r(x)
=
2\frac{\sinh^2(N(x))}{\cosh^2(x)}
\le
2\tanh^2(N(x)),
\]
where we used the~assumption $N(x)\le x$. Therefore,
\[
\frac{2}{2+r(x)} \ge \frac{1}{1+\tanh^2(N(x))},
\]
which yields the~desired lower bound.
\end{proof}

\begin{lemma}\label{lemma:bound_on_-B}
Let $0<N(x) \leq 1$ for~$x>0$. Then for~all $x>0$,
\[
2x^2-\frac12\frac{\cosh(2x)-1}{\cosh(2x)+\cosh(2N(x))}\bigl(x^2+1+\cosh(2N(x))\bigr)
\ge
\frac{9}{10}\,x^2.
\]
\end{lemma}

\begin{proof}
By \eqref{eq:identities_cosh^2_sinh^2}, we have
\[
\frac{\cosh(2x)-1}{\cosh(2x)+\cosh(2N(x))} = \frac{\sinh^2(x)}{\cosh^2(x) + \sinh^2(N(x))},
\]
and therefore
\[
\frac12\frac{\cosh(2x)-1}{\cosh(2x)+\cosh(2N(x))}\bigl(x^2+1+\cosh(2N(x))\bigr)
=
\frac{\sinh^2(x)}{\cosh^2(x) + \sinh^2(N(x))} \left(\frac12 x^2 + \cosh^2(N(x)) \right).
\]
Using
\begin{equation}\label{eq:cosh^2-sinh^2=1}
    \cosh^2(x) - \sinh^2(x) = 1,
\end{equation}
a direct rearrangement yields
\[
2x^2-\frac12\frac{\cosh(2x)-1}{\cosh(2x)+\cosh(2N(x))}\bigl(x^2+1+\cosh(2N(x))\bigr)
=
\frac32 x^2 - \frac{\cosh^2(N(x))}{\cosh^2(x) + \sinh^2(N(x))} \bigl( \sinh^2(x) - \frac 12 x^2 \bigr).
\]
Since $x\le\sinh(x)$, we have $\sinh^2(x) - \frac12 x^2 \ge 0$, so it suffices to~show
\begin{equation}\label{eq:(Lx)^2_to_prove}
\frac{\cosh^2(N(x))}{\cosh^2(x) + \sinh^2(N(x))} \bigl( \sinh^2(x) - \frac12 x^2 \bigr)
\le
\frac 35 x^2.
\end{equation}

Using \eqref{eq:cosh^2-sinh^2=1} again, this is equivalent to
\[
0
\le
\frac{11}{10}x^2 \cosh^2(N(x)) + \sinh^2(x)\left(\frac35x^2-\cosh^2(N(x))\right).
\]
If $x^2 \ge \frac53 \cosh^2(N(x))$, then both terms are nonnegative and~the~claim follows. Hence, assume $x^2 < \tfrac 53 \cosh^2(N(x))$.
It suffices to~prove
\begin{equation}\label{eq:final_(Lx)^2_to_prove}
\frac{\sinh^2(x)}{x^2}
\le
\frac{\frac{11}{10}\cosh^2(N(x))}{\cosh^2(N(x)) - \frac 35 x^2}
\qquad \text{for } x^2 \le \frac53\cosh^2(N(x)).
\end{equation}
Since $\cosh^2(N(x))\le \cosh^2(1) < \frac{12}{5}$, we have for~$x^2 \le \frac53\cosh^2(N(x))$,
\[
\frac{\frac{11}{10}\cosh^2(N(x))}{\cosh^2(N(x)) - \frac 35 x^2}
\ge
\frac{22}{5(4-x^2)}.
\]
Note that $|x|\le 2$. On the~other hand, using the~Taylor expansion
\[
\frac{\sinh^2(x)}{x^2}
= 1+\frac13x^2 + \frac{2}{45}x^4 + \frac{1}{315}x^6+\cdots,
\]
all coefficients are nonnegative, hence the~remainder
\[
\frac{\sinh^2(x)}{x^2} - 1-\frac{1}{3}x^2
\]
is increasing for~$x\ge 0$. Therefore, the~function
\[
x \mapsto \frac{\frac{\sinh^2(x)}{x^2} - 1-\frac{1}{3}x^2}{x^4}
\]
attains its supremum on~$[0,\sqrt{5/3}\cosh(N(x))]$ at~the~boundary. A~direct evaluation gives
\[
\beta := \sup_{|x|\le 2}
\frac{\frac{\sinh^2(x)}{x^2} - 1-\frac{1}{3}x^2}{x^4}
< \frac{1}{16}.
\]
Hence
\[
\frac{\sinh^2(x)}{x^2} \le 1 +\frac{1}{3}x^2 + \frac{1}{16}x^4
\qquad \forall |x|\le \sqrt{\frac53}\cosh(N(x)) \le 2.
\]

It remains to~show
\[
1 +\frac{1}{3}x^2 + \frac{1}{16}x^4 \le \frac{22}{5(4-x^2)},
\]
which is equivalent to
\[
p(x):=-\frac{1}{16}x^6 - \frac{1}{12}x^4 + \frac{1}{3}x^2 - \frac25 \le 0.
\]
A direct computation shows that $p'(x)=0$ iff $x=0$ or~$x^2=\frac{4}{9}(\sqrt{10}-1)$, so the~only positive critical point is $x_0=\left(\frac{4}{9}(\sqrt{10}-1)\right)^{1/2}$. Evaluating $p$ at~$x_0$, we obtain $p(x_0)<-0.21$.
Thus $p(x)\le 0$.

Hence \eqref{eq:final_(Lx)^2_to_prove} holds, which implies
\eqref{eq:(Lx)^2_to_prove} and~completes the proof.
\end{proof}

\subsection{Completing the~final lower bound}\label{subsection:final_lower_bound}

In this subsection we combine the~results of~the~previous subsections to~complete the~proof of~Theorem~\ref{thm:main_result_full_statement}. In~particular, we show that the~coarse bump constructed in~Subsection~\ref{subsection:bump_analysis} can be sharpened by~iterating $\sigma_a$ with $a>1$, as studied in~Subsection~\ref{subsection:properties_of_sigma_n}. This refinement yields positivity on~the~mass cube and~exponential convergence to~the~fixed point $-x_\star$ outside the~region of~interest. The~overall procedure is illustrated in~Figure~\ref{fig:sharpening}.

\begin{theorem}[Bump sharpening]\label{thm:num_of_samples_with_bump_property}
Let $n,q,d,B\in\N$ and~$\kappa\in(0,1)$ with $c>0$ with $\tilde c :=B^{1-2/q}c>1$ and~$b:=\tanh\left(\frac c2 \right)$. Let $\bar x_{\tilde c},\tau_{\tilde c}, \xi_{n, \tilde c}, \e_{n, \tilde c}$ be associated with slope $\tilde c$. Define moreover
\[
\xi_{\frac{c}{rB^{1/q}}}
:=
\frac{c}{rB^{1/q}}\left(1-\left(\bar x_{\tilde c}-\tau_{\tilde c}\right)^2\right)\,,
\qquad
\tilde\xi_n
:=
\xi_{\frac{c}{rB^{1/q}}}^{-1}\,
\xi_{n, \tilde c} \, .
\]
Let
\[\alpha = \frac35 \cdot \frac{\tanh(B^{-1/q}(c-b))}{B^{-1/q}(c-b)\sqrt d}\]
and
\begin{equation}\label{eq:assumption_xi_n_theorem}
    \tilde \xi_n \le \frac{\tanh(2B^{-1/q}b)}{2d}\tanh^2(B^{-1/q}(c-b)).
\end{equation}
If
\begin{equation}\label{eq:necessary_edge_length}
\sqrt{\frac{10(1+\cosh(2B^{-1/q}b))}{9\tanh(B^{-1/q}b)(c-b)^2} \, B^{2/q} \,d\,\tilde\xi_n} \leq M \leq 1,
\end{equation}
then it holds:

\begin{enumerate}
\item for~all $x\in \alpha[-M,M]^d$,
\[
\sigma_{\tilde c}^{\,n} \circ \sigma_{\frac{c}{rB^{1/q}}}\!\left(\frac 1d\sum_{i=1}^d f(x_i)- f(\alpha M)\right)\;\ge\;0,
\]

\item for~all $x\notin [-M,M]^d$,
\[
\left|
\sigma_{\tilde c}^{\,n} \circ \sigma_{\frac{a}{rB^{1/q}}}\!\left(\frac 1d\sum_{i=1}^d f(x_i)-f(\alpha M)\right)
+ x_{\star,\tilde c}
\right|
\le \e_{n, \tilde c}.
\]
\end{enumerate}
\end{theorem}

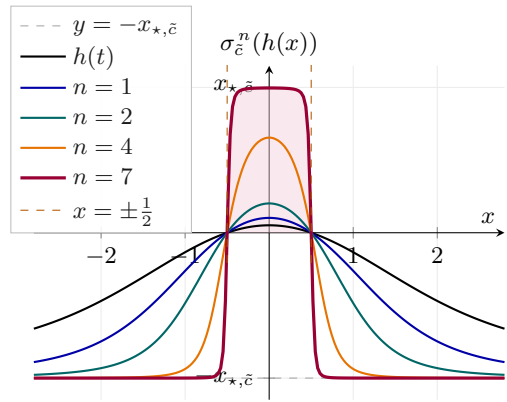
\begin{figure}[b]
\centering
\begin{tikzpicture}
\begin{axis}[
    width=7.8cm, height=6cm,
    xlabel={$x$},
    ylabel={$\sigma_{\tilde c}^{\,n}(h(x))$},
    axis lines=center,
    xtick={-2,-1,0,1,2},
    ytick={-0.9575, 0, 0.9575},
    yticklabels={$-x_{\star, \tilde c}$, $0$, $x_{\star, \tilde c}$},
    ylabel style={at={(0.5,1)}, anchor=south},
    ymax=1.1, ymin = -1.1,
    tick label style={font=\footnotesize},
    label style={font=\small},
    legend style={
        at={(-0.05,1.18)}, anchor=north west,
        font=\footnotesize, draw=gray!50,
        fill=white, fill opacity=0.9
    },
    legend cell align=left,
    grid=both,
    grid style={gray!15},
    clip=true,
]

\addplot[black!35, dashed, thin, domain=-2.8:2.8, samples=200] {-0.9575};
\addlegendentry{$y=- x_{\star, \tilde c}$}

\addplot[fill=purple!20!, fill opacity=0.45, draw=none, domain=-0.5:0.5, samples=50, forget plot] {tanh(2*tanh(2*tanh(2*tanh(2*tanh(2*tanh(2*tanh(2*( 
tanh((1-tanh(0.5))*x + tanh(0.5))
    - tanh((1-tanh(0.5))*x - tanh(0.5))
    - tanh((1-tanh(0.5))*0.5 + tanh(0.5))
    + tanh((1-tanh(0.5))*0.5 - tanh(0.5))
))))))))} \closedcycle;

\addplot[black, thick, domain=-2.8:2.8, samples=200]
    { tanh((1-tanh(0.5))*x + tanh(0.5))
    - tanh((1-tanh(0.5))*x - tanh(0.5))
    - tanh((1-tanh(0.5))*0.5 + tanh(0.5))
    + tanh((1-tanh(0.5))*0.5 - tanh(0.5)) };
\addlegendentry{$h(t)$}

\addplot[blue!65!black, thick, domain=-2.8:2.8, samples=200]
    {tanh(2*(
    tanh((1-tanh(0.5))*x + tanh(0.5))
    - tanh((1-tanh(0.5))*x - tanh(0.5))
    - tanh((1-tanh(0.5))*0.5 + tanh(0.5))
    + tanh((1-tanh(0.5))*0.5 - tanh(0.5))
    ))};
\addlegendentry{$n=1$}

\addplot[teal!80!black, thick, domain=-2.8:2.8, samples=200]
    {tanh(2*tanh(2*(
    tanh((1-tanh(0.5))*x + tanh(0.5))
    - tanh((1-tanh(0.5))*x - tanh(0.5))
    - tanh((1-tanh(0.5))*0.5 + tanh(0.5))
    + tanh((1-tanh(0.5))*0.5 - tanh(0.5))    
    )))};
\addlegendentry{$n=2$}

\addplot[orange!90!black, thick, domain=-2.8:2.8, samples=200]
    {tanh(2*tanh(2*tanh(2*tanh(2*(
    tanh((1-tanh(0.5))*x + tanh(0.5))
    - tanh((1-tanh(0.5))*x - tanh(0.5))
    - tanh((1-tanh(0.5))*0.5 + tanh(0.5))
    + tanh((1-tanh(0.5))*0.5 - tanh(0.5))
    )))))};
\addlegendentry{$n=4$}

\addplot[purple!80!black, very thick, domain=-2.8:2.8, samples=200]
    {tanh(2*tanh(2*tanh(2*tanh(2*tanh(2*tanh(2*tanh(2*(
    tanh((1-tanh(0.5))*x + tanh(0.5))
    - tanh((1-tanh(0.5))*x - tanh(0.5))
    - tanh((1-tanh(0.5))*0.5 + tanh(0.5))
    + tanh((1-tanh(0.5))*0.5 - tanh(0.5))
    ))))))))};
\addlegendentry{$n=7$}

\addplot[orange!70!black, dashed, thin, forget plot] coordinates {(0.5,-0.2)(0.5,1.7)};
\addplot[orange!70!black, dashed, thin] coordinates {(-0.5,-0.2)(-0.5,1.7)};
\addlegendentry{$x=\pm \frac12$}

\end{axis}
\end{tikzpicture}
\caption{
Illustration of~the~sharpening mechanism for~$d=1$.
Let $c = 1$ with $b=\tanh(c/2)$ and~$\tilde c=2$ with the~fixed point $x_{\star,\tilde c}\approx0.96$. Let $f(x)=f(x \, | \, 1, c, b)$ and~$h(x)=f(x)-f(\tfrac12)$
so that $h(\pm\tfrac12)=0$ and~$h(x)>0$ for~$|x|<\tfrac12$.
Applying the~iterates $\sigma_{\tilde c}^{\,n}$ (see Figure~\ref{fig:sigma_n_step_function}) to~$h$ produces a~localized bump: positive on~$[-\tfrac12,\tfrac12]$ and~converging to~$x_{\star, \tilde c}$, while approaching $-x_{\star, \tilde c}$ exponentially fast outside as $n$ increases.
}
\label{fig:sharpening}
\end{figure}

\begin{proof}
By \eqref{eq:der_of_f}, $f$ is decreasing on~$(0,\infty)$, hence on~$\alpha[-M,M]^d$ the~minimum of~$\frac{1}{d}\sum f(x_i)$ is attained at~the~boundary. Thus,
\[
\frac{1}{d}\sum_{i=1}^d f(x_i)-f(\alpha M) \ge f(\alpha M)-f(\alpha M)=0,
\]
which proves the~first claim.

Now let $x\notin[-M,M]^d$. Then the~maximum of~the~average is attained when one coordinate equals $M$ and~the~others are zero, so
\[
\frac{1}{d}\sum f(x_i)-f(\alpha M)
\le \frac{d-1}{d}f(0)+\frac{1}{d}f(M)-f(\alpha M).
\]
We show that the~right-hand side is bounded by~$-\tilde\xi_n$ by~verifying the~assumptions of~Lemma~\ref{lemma:condition_on_alpha_and_M}.

\textbf{Verification of~the~assumption on~$\tilde\xi_n$.}
We use the~same notation as in~Lemma~\ref{lemma:condition_on_alpha_and_M} and~set
\[
A:=B^{-1/q}(c-b),
\qquad
C:=B^{-1/q}b,
\qquad
K=\cosh(2C).
\]
Since $2\tanh(c/2)\le c$, we have $C\le A$. Therefore, Claim~\ref{claim:bounds_on_cosh/cosh} applies with $x:=A$ and~$N(x):=C$. Hence, using the~identity $\tanh(2y)=\frac{2\tanh(y)}{1+\tanh^2(y)}$, we obtain
\[
\frac{\tanh(2C)}{2d}\tanh^2(A)
=
\frac{1}{d}\cdot\frac{\tanh(C)}{1+\tanh^2(C)}\tanh^2(A)
\le
\frac{\tanh(C)}{d}\cdot
\frac{\cosh(2A)-1}{\cosh(2A)+\cosh(2C)}.
\]
Thus, the~condition~\eqref{eq:assumption_xi_L} of~Lemma~\ref{lemma:condition_on_alpha_and_M} is ensured by~\eqref{eq:assumption_xi_n_theorem}.

\textbf{Verification of~the~assumption on~$\alpha$.}
Define
\[
\beta
=
\frac{1+\cosh(2C)}{2d}\cdot
\frac{\cosh(2A)-1}{\cosh(2A)+K},
\qquad
y:=\arccosh(1+\beta).
\]
Since $C<1$, we obtain from \ref{eq:cosh/cosh<tanh^2} that
\[
\beta < \frac{1+\cosh(2)}{2}.
\]
Hence, analogously to~\eqref{eq:quadratic_bound_cosh}, we have
\[
\cosh(y)\le 1+\left(\frac{5}{6}\right)^2 y^2,
\]
which implies
\[
\frac{6}{5}\sqrt{\beta} \le \arccosh(1+\beta).
\]
Moreover, using again \eqref{eq:cosh/cosh<tanh^2} and~\eqref{eq:identities_cosh^2_sinh^2}, we obtain
\[
\frac{\cosh^2(C)}{1+\tanh^2(C)}\cdot\frac{\tanh^2(A)}{d}
\le \beta.
\]
By standard calculus,
\[
1 \le \frac{\cosh^2(C)}{1+\tanh^2(C)} \le \frac{\cosh^2(1)}{1+\tanh^2(1)} \le 1.51,
\]
and therefore
\[
\alpha = \frac{3}{5}\cdot\frac{\tanh(A)}{A\sqrt d}
\le \frac{1}{2A}\arccosh(1+\beta),
\]
which verifies assumption~\eqref{eq:assumption_on_alpha}.
Therefore, all assumptions of~Lemma~\ref{lemma:condition_on_alpha_and_M} are satisfied, and~together with \eqref{eq:necessary_edge_length} we obtain
\[
\frac{1}{d}\sum f(x_i)- f(\alpha M) \le -\tilde\xi_n.
\]

\textbf{Basin of~attraction property.}
We now aim to~apply Lemma~\ref{lemma:error_of_basin_of_attraction}. Set
\[
u := \frac{1}{d}\sum_{i=1}^d f(x_i)-f(\alpha M).
\]
Since we already proved that $u\le -\tilde\xi_n<0$, it remains to~estimate the~convergence of
$\sigma_{\frac{c}{rB^{1/q}}}(u)$ under the~iterates of~$\sigma_{\tilde c}$.
We therefore show that $\sigma_{\frac{c}{rB^{1/q}}}(u)$ already lies in~the~basin of~attraction of~$\sigma_{\tilde c}$.

Since
\[
I_{\tilde c}
=
(-\bar x_{\tilde c}+\tau_{\tilde c},\,\bar x_{\tilde c}-\tau_{\tilde c}),
\]
by the~definition of~$\xi_{n,\tilde c}$ we have
\[
\xi_{n,\tilde c}\le \bar x_{\tilde c}-\tau_{\tilde c}.
\]
Moreover, since
$\sigma^{-1}_{\frac{c}{rB^{1/q}}}(0)=0$,
the mean value theorem yields
\[
\sigma^{-1}_{\frac{c}{rB^{1/q}}}(\xi_{n,\tilde c})
\le
\sup_{z\in I_{\tilde c}}
\left(\sigma^{-1}_{\frac{c}{rB^{1/q}}}(z)\right)'
\,\xi_{n,\tilde c}.
\]
Using the~identity
\[
\left(\sigma^{-1}_{\frac{c}{rB^{1/q}}}(z)\right)'
=
\frac{1}{
\sigma'_{\frac{c}{rB^{1/q}}}
\!\left(
\sigma^{-1}_{\frac{c}{rB^{1/q}}}(z)
\right)},
\]
together with \eqref{eq:derivation_of_g}, we obtain
\[
\sigma'_{\frac{c}{rB^{1/q}}}
\!\left(
\sigma^{-1}_{\frac{c}{rB^{1/q}}}(z)
\right)
=
\frac{c}{rB^{1/q}}
\left(
1-z^2
\right)
\ge
\frac{c}{rB^{1/q}}
\left(
1-(\bar x_{\tilde c}-\tau_{\tilde c})^2
\right)
=
\xi_{\frac{c}{rB^{1/q}}}
\]
for all
$z\in I_{\tilde c}$.
Hence,
\[
\sigma^{-1}_{\frac{c}{rB^{1/q}}}(\xi_{n,\tilde c})
\le
\xi_{\frac{c}{rB^{1/q}}}^{-1}
\xi_{n,\tilde c}
=
\tilde\xi_n.
\]

Since $\sigma_{\frac{c}{rB^{1/q}}}$ is increasing and~odd, we conclude
\[
u\le -\tilde\xi_n
\le
-
\sigma^{-1}_{\frac{c}{rB^{1/q}}}(\xi_{n,\tilde c})
\qquad\Longrightarrow\qquad
\sigma_{\frac{c}{rB^{1/q}}}(u)
\le
-\xi_{n,\tilde c}.
\]
Hence, $\sigma_{\frac{c}{rB^{1/q}}}(u)$ already belongs to~the~basin of~attraction of~$\sigma_{\tilde c}$, and~therefore Lemma~\ref{lemma:error_of_basin_of_attraction} yields
\[
\left|
\sigma_{\tilde c}^{\,n}
\circ
\sigma_{\frac{c}{rB^{1/q}}}(u)
+
x_{\star,\tilde c}
\right|
\le
\e_{n, \tilde c} \, ,
\]
which proves item~2.

\end{proof}

\begin{remark}
Note that since $\frac1r < B^{1-\frac1q}$, we have $\frac{c}{rB^{1/q}}
<
B^{1-2/q}c,$
so the~derivative $\xi_{\frac{c}{rB^{1/q}}}$ may be smaller than one, which can unfortunately increase the~threshold $\tilde\xi_n$ required to enter the basin of attraction of $\sigma_{\tilde c}$.
\end{remark}

Using the~results of~the~previous subsections together with Theorem~\ref{thm:num_of_samples_with_bump_property}, we obtain the~following constructive result. It shows that one can explicitly construct a~localized $\tanh$ neural network which has positive mass on~a~cube of~width $2M$ and~is exponentially small outside this region, even for~potentially exponentially small values of~$M$.

\begin{theorem}[Construction of~bump networks]\label{thm:existence_of_Phi_for_M>}
Let $p,q,d,m,B,L,j,k\in\N$ and~$c>0$ satisfy $B\ge 2d$, $k\geq3$ and~$\tilde c :=B^{1-\frac{2}{q}}c>1$. Let $j$ satisfy
\begin{equation}\label{eq:assumption_on_j}
\frac{\ln\left( \frac{5 B^{\frac3q}\arccosh\left(\sqrt{\tilde c}\right)\,\rho(\tilde c)}{c^2\tanh(2B^{-1/q}\tanh(\frac c2))\tanh^2\left[B^{-1/q}(c - \tanh(\frac c2))\right]} \right)}
{\ln\left(  \frac
{\tilde c}
{\cosh^2\left[2\arccosh(\sqrt{\tilde c})\rho(\tilde c)\right]} \right)
}
\leq
j
\end{equation}
and $j+k=L-3$. Define constants
\[
\Theta := \frac{9\,\tanh\left(B^{-1/q}\tanh\left(\frac c2\right)\right)}{50\,\cosh^2\left(B^{-1/q}\tanh\left(\frac c2\right)\right)}
\qquad 
\text{and}
\qquad
\Omega := \frac35 \cdot \frac{\tanh\left(B^{-1/q}\left(c-\tanh\left(\frac c2\right)\right)\right)}{B^{-1/q}\left(c-\tanh\left(\frac c2\right)\right)}.
\]

Then for~all $y_\ell \in [0,1]^d$ and~$0<M \le1$ satisfying
\begin{equation}\label{eq:minimal_M_final_bound}
M
\ge 
\left(\Theta\cdot\frac{c^2\left(c-\tanh\!\left(\frac c2\right)\right)^2}{B^{\frac 5q}\arccosh\!\left(\sqrt{\tilde c}\right)\,\rho(\tilde c)}\right)^{-\frac 12}
\left(\frac
{\tilde c}
{\cosh^2\left[2\arccosh(\sqrt{\tilde c})\rho(\tilde c)\right]}\right)^{-\frac j2},
\end{equation}
then for~every $y_\ell \in [0,1]^d$ and~$\nu\in\{\pm1\}$, there exists a~$\tanh$ neural network $\Phi_{y_\ell,\nu}\colon \R^d\to\R$ of~width $B$ and~depth $L$ with $\|\Phi_{y_\ell,\nu}\|_{\ell^q} \le c$ such that:
\begin{enumerate}
\item $\Phi$ has positive $L^p$ mass, i.e.
\begin{equation}\label{eq:positive_norm_thm_M}
\|\Phi_{y_\ell,\nu}\|_{L^p([0,1]^d)} 
\ge 
c\frac{\sqrt{\tilde c^2-1}}{\tilde c} \cdot \left(\frac{2 \, \Omega \, M}{\sqrt d}\right)^{\frac{d}{p}}.
\end{equation}

\item at~the~same time, for~all $x\notin y_\ell + [-M,M]^d$, it holds that
\begin{equation*}\label{eq:exponential_small_error_thm_M}
    |\Phi_{y_\ell,\nu}(x)|
    \le
    \frac{c\,\tilde c^{k-2}}
    {\pi(\tilde c)\left(\frac{1+\pi(\tilde c)^{-1}}{2}\right)^2\,\cosh^{\,2( k-3)}\!\left( \tilde c \cdot\frac{\pi(\tilde c) - 1}{\pi(\tilde c )+1} \right)}
    = 
    O\!\left(\frac{c\,\tilde c^{k-4}}{e^{2(k-3)\tilde c}}\right).
\end{equation*}
\end{enumerate}
\end{theorem}

\begin{proof}

\textbf{Construction of~the~network.}
We construct $\Phi_{y_\ell,\nu}$ as described in~Figure~\ref{fig:network-architecture}. First, we build a~coarse bump centered around $y_\ell$, which is then sharpened by~$\sigma_{\tilde c}^{L-3}$, which finally results in~a localized bump network $\Phi_{y_\ell,\nu}$ with the~bump sign of~$\nu$. In~the~first layer we use $2d$ neurons with weight matrix and~bias vector
\begin{equation*}
W_1 = \frac{c-b}{B^{\frac1q}}\cdot
\begin{pmatrix}
1 &0 & 0 & \cdots & 0 \\
-1 &0 & 0 & \cdots & 0 \\
0 & 1 & 0 & \cdots & 0 \\
0 & -1 & 0 & \cdots & 0 \\
\vdots & & \ddots & & \vdots \\
0 & 0 & 0 &  \cdots & 1 \\
0 & 0 & 0 &\cdots & -1
\end{pmatrix}
\in \R^{2d\times d}
\qquad
\text{and}
\qquad
b_1 
=
\frac{1}{B^{\frac1q}}\cdot
\begin{pmatrix}
    -(c-b)\,(y_\ell)_1  + b \\
    (c-b)\,(y_\ell)_1  + b \\
    -(c-b)\,(y_\ell)_2  + b \\
    (c-b)\,(y_\ell)_2  + b \\
    \vdots\\
    -(c-b)\,(y_\ell)_d  + b \\
    (c-b)\,(y_\ell)_d  + b \\
\end{pmatrix}
\in \R^{2d}
\end{equation*}
for
\[b:=\tanh\left(\frac c2 \right) < 1.\]
Since $y_\ell \in[0,1]^d$, we have $\|W_1\|_{\ell^q} \leq c$ and~$\|b_1\|_{\ell^q} \leq \|\Phi\|_{\ell^q}$. Moreover, note that $\frac c2 < c-b < c$.
The layer output is then
\begin{equation*}
\begin{pmatrix}
\sigma\left(\frac{c-b}{B^{1/q}}(x_1-(y_\ell)_1) + \frac{b}{B^{1/q}}\right) \\[5pt]
-\sigma\left(\frac{c-b}{B^{1/q}}(x_1-(y_\ell)_1) - \frac{b}{B^{1/q}}\right) \\[5pt]
\vdots \\
\sigma\left(\frac{c-b}{B^{1/q}}(x_d - (y_\ell)_d) + \frac{b}{B^{1/q}}\right) \\[5pt]
-\sigma\left(\frac{c-b}{B^{1/q}}(x_d - (y_\ell)_d) - \frac{b}{B^{1/q}}\right)
\end{pmatrix}.
\end{equation*}

In the~following layer we use $B$ neurons with weights
\[
W_2 
= 
\frac{c}{B^{1+\frac1q}}\, 1_{B\times 2d} \in \R^{B\times 2d}
\qquad
\text{and}
\qquad
b_2 = - \frac{c}{rB^{\frac1q}} \cdot f\!\left(\alpha\, M\right)\, 1_{ B \times 1 } \in \R^{B}.
\]
Clearly, $\|W_2\|_{\ell^q} = c B^{\frac1q - 1} \leq c$ and~$\|b_2\|_{\ell^q} \leq2\frac{c}{r} \tanh\left(\frac{b}{B^{1/q}}\right) \leq c.$ Using the~notation of~Table~\ref{tab:definitions}, denote
\[f(y) := \tanh\left(B^{-1/q}\big((c-b)y + b\big)\right) - \tanh\left(B^{-1/q}\big((c-b)y - b\big)\right).\]
The output of~this layer is, therefore,
\begin{equation*}
    \sigma\!\left( \frac{c}{rB^{1/q}} \begin{pmatrix}
\frac 1d \sum_{i=1}^d f(x_i-(y_\ell)_i) - f\!\left(\alpha M\right) \\
\vdots \\
\frac 1d \sum_{i=1}^d f(x_i-(y_\ell)_i) - f\!\left(\alpha M\right)
\end{pmatrix}\right).
\end{equation*}

In each of~the~subsequent layers we apply the~same weight matrix and~zero bias,
\[
W_i = \frac{c}{B^{\frac2q}} \cdot 1_{B\times B}  \in \R^{B\times B},
\qquad 
b_i = 0\cdot 1_{B\times 1},
\qquad 
i\in\{3,\dots,L-1\},
\]
together with the~$\tanh$ activation. By~repeated composition this yields
\begin{equation*}
    \begin{pmatrix}
\sigma_{\tilde c}^{\,L-3}\!\left( \sigma_{\frac{c}{rB^{1/q}}}\Bigl(\frac1d \sum_{i=1}^d f(x_i-(y_\ell)_i) - f(\alpha M)\Bigr) \right) \\
\vdots \\
\sigma_{\tilde c}^{\,L-3}\!\left( \sigma_{\frac{c}{rB^{1/q}}}\Bigl(\frac1d \sum_{i=1}^d f(x_i-(y_\ell)_i) - f(\alpha M)\Bigr) \right) 
\end{pmatrix}.
\end{equation*}

In the~final layer we select a~single coordinate and~shift it by~the~fixed point value. Concretely, we take
\[
W_{L} = c\,\nu \cdot
\begin{pmatrix}
    1&0&\dots&0
\end{pmatrix} \in\R^{1\times B}
\qquad 
\text{and}
\qquad
b_{L}= c\,\nu \, x_{\star, \tilde c},
\]
where $\nu\in\{\pm1\}$ and~$x_{\star, \tilde c}$ is 
the fixed point satisfying $\tanh(B\tilde c\cdot x_{\star, \tilde c})=x_{\star, \tilde c}<1$. Hence, we verified that $\|\Phi_{y_\ell,\nu}\|_{\ell^q} \le c$. 

Then, we naturally aim to~use Theorem~\ref{thm:num_of_samples_with_bump_property} with $n=L-3$. We choose $\kappa\in(0,1)$ such that
\[
\lceil \kappa n\rceil = j,
\qquad
\lfloor (1-\kappa)n\rfloor = k.
\]
This choice enforces the~decomposition $j+k=n$, which is assumed. Then any
\[
\kappa \in \left(\frac{j-1}{L-3},\,\frac{j}{L-3}\right]
\]
satisfies both relations.

Afterwards, we choose the~maximal $\tau_{\tilde c}$ from~Claim~\ref{claim:tau_choosing},
\[
\tau_{\tilde c}=\bar x_{\tilde c}(1-2\rho(\tilde c)),
\qquad
\bar x_{\tilde c}=\frac{1}{\tilde c}\arccosh(\sqrt{\tilde c}),
\]
where $\sigma_{\tilde c}'(\bar x_{\tilde c})=1$. Then, by~Corollary~\ref{col:lower_bound_xi},
\[
\tilde \xi_{L-3}
=
\frac{\left(\cosh^2\!\left[2\tilde c\bar x_{\tilde c}\rho(\tilde c)\right]\right)^{j}rB^{1/q}}
{\tilde c^{j} c}
\cdot \frac{2\bar x_{\tilde c}\rho(\tilde c)}{1-\left(2\bar x_{\tilde c}\rho(\tilde c)\right)^2}.
\]
Using Lemma~\ref{lemma:arccosh_upper_bound}, we have
\[
2\bar x_{\tilde c}\rho(\tilde c) \leq \frac{2\sqrt{\tilde c-1}}{\tilde c\left(\tilde c \mu + 1\right)}
\]
for $\mu = \sqrt{1-4\frac{31^2}{100^2}}$.
The function on~right-hand side is maximized at~$\tilde c_0=\frac{4\mu -1 + \sqrt{16\mu^2+16\mu+1}}{6\mu}$, which gives $2\bar x_{\tilde c}\rho(\tilde c) \leq\frac{11}{25}$. Hence,
\[
\frac{y}{1-y^2}
\leq
\frac{5}{4}y \quad \text{for } 0\leq y \leq \frac{11}{25}.
\]
Thus,
\[
\tilde \xi_{L-3}
\le
\frac{5}{2}
\frac{\cosh^{2j}\!\left[2\tilde c\bar x_{\tilde c}\rho(\tilde c)\right]}
{\tilde c^{j+1} c}
\cdot rB^{\frac1q}\arccosh(\sqrt{\tilde c})\,\rho(\tilde c).
\]
We now verify the~assumptions of~Theorem~\ref{thm:num_of_samples_with_bump_property}.

\textbf{Verification of~the~assumptions.}
We set $\alpha:= \Omega / \sqrt{d}$. By~assumption~\eqref{eq:assumption_on_j}, the~above bound on~$\tilde \xi_{L-3}$ ensures that condition~\eqref{eq:assumption_xi_n_theorem} holds (using $B=rd$). It remains to~verify the~edge-length condition~\eqref{eq:necessary_edge_length}. A~direct computation shows that it follows from
\[
M^2
\geq
\frac{20\,B^{\frac2q}\cosh^2\big(B^{-\frac1q}\tanh(\frac c2)\big)}{9\,\left(c-\tanh\left(\frac c2\right)\right)^2\tanh\big(B^{-\frac 1q}\tanh(\frac c2)\big)}
\cdot
\frac52
\frac{\cosh^{2j}\!\left[2\tilde c\bar x_{\tilde c}\rho(\tilde c)\right]}
{ \tilde c^{j+1} c}
\cdot
B^{1+\frac1q}\arccosh\big(\sqrt{\tilde c}\big)\,\rho\left(\tilde c\right),
\]
and this is guaranteed by~the~choice of~$M$ in~\eqref{eq:minimal_M_final_bound}. Hence all assumptions of~Theorem~\ref{thm:num_of_samples_with_bump_property} are satisfied.

Hence, by~Theorem~\ref{thm:num_of_samples_with_bump_property}, we obtain that for~all
\[
x\in y_\ell + \frac{\Omega}{\sqrt{d}} [-M, +M]^d,
\]
it holds that $|\Phi_{y_\ell,\nu}| \geq c \, x_{\star, \tilde c}$. Moreover, using \eqref{eq:bounds_on_tanh}, we obtain
\[
x_{\star, \tilde c}
=
\tanh(\tilde c\,x_{\star, \tilde c})
\ge
\frac{\tilde c\,x_{\star, \tilde c}}{\sqrt{1+(\tilde c\,x_{\star, \tilde c})^2}},
\]
which implies
\[
x_{\star, \tilde c}
\ge
\frac{\sqrt{\tilde c^2-1}}{\tilde c}.
\]
The lower bound on~the~norm now follows from~the~construction:
\[
\|\Phi_{y_\ell,\nu}\|_{L^p(\R^d)}^p
\geq
c^p x_{\star, \tilde c}^p \cdot \mu_d\left( \frac{\Omega}{\sqrt{d}} \cdot \left[ -M , M \right]^d \right),
\]
where $\mu_d$ denotes the~Lebesgue measure on~$\R^d$.

From the~second item of~Theorem~\ref{thm:num_of_samples_with_bump_property}, we also obtain that for~all
$x\notin y_\ell + [-M, +M]^d$,
\[
|\Phi_{y_\ell,\nu}(x)| \le c\, \e_{L-3, \tilde c}.
\]
Moreover, from~Corollary~\ref{col:lower_bound_xi_simplier}, we get
\[
\e_{L-3, \tilde c}
=
\eta_{\tilde c} \bar\eta_{\tilde c} ^{k -3} 
\leq
\frac{\tilde c^{k -2}}
{\pi(\tilde c)\left(\frac{1+\pi(\tilde c)^{-1}}{2}\right)^2\left( \cosh\left( \tilde c \cdot \frac{\pi(\tilde c) -1}{\pi(\tilde c)+1} \right) \right)^{2(k-3)}}.
\]
Multiplying by~$c$ (from the~outer scaling in~the~construction) yields the~stated bound.
\end{proof}

We combine the~localized bump construction with the~grid packing argument to~prove Theorem~\ref{thm:main_result_full_statement}. Choosing $M \asymp m^{-1/s}$, the~exponential decay from~Theorem~\ref{thm:existence_of_Phi_for_M>} ensures that all tail contributions can be pushed below machine precision $\e_p$. This yields a~family of~effectively disjoint bump networks. We then apply a~standard volume packing argument together with the~Markov-type technique from~\cite{berner2023learning} to~obtain the~final lower bound.

\begin{proof}[Proof of~Theorem~\ref{thm:main_result_full_statement}]\label{proof:main_thm}
We define $k=\lceil (4m)^{1/s} \rceil$ and~$M:=\frac{1}{2k}$. Set grid centers
\[
y_\ell = \frac{(1,\dots,1)}{2k} + \frac{\ell - (1,\dots,1)}{k}
\in [0,1]^d 
\qquad \text{for } \ell \in [k]^d.
\]
Since $m\le m_{\max}$, we have
\[
M \ge \frac{1}{4(4m_{\max})^{1/s}}.
\]
Hence, by~\eqref{eq:assumption_on_s_m_max}, all assumptions of~Theorem~\ref{thm:existence_of_Phi_for_M>} are satisfied with dimension parameter $d:=s$.
Moreover, by~\eqref{eq:assumption_on_k_num_of_lyers}, we have
\[
\frac{c\,\tilde c^{k-2}}
{\pi(\tilde c)\left(\frac{1+\pi(\tilde c)^{-1}}{2}\right)^2
\cosh^{2(k-3)}\!\left(\tilde c \cdot\frac{\pi(\tilde c)-1}{\pi(\tilde c)+1}\right)}
\le \e_p.
\]
Thus, for~every $y_\ell$ and~$\nu\in\{\pm1\}$, there exists a~neural network $\Phi_{y_\ell,\nu}$ centered at~$y_\ell$ such that
\begin{itemize}
    \item it exhibits a~positive/negative bump,
    \item it satisfies the~lower bound \eqref{eq:positive_norm_thm_M},
    \item and~outside $y_\ell + [-M,M]^d$ it is smaller than machine precision $\e_p$.
\end{itemize}

By the~grid construction (spacing $2M$), the~supports above machine precision are disjoint: for~all $\ell\neq\ell'$ in
\[
\Gamma := [k]^s \times \{1\}^{d-s} \subset [k]^d,
\]
we have
\begin{equation}\label{eq:different_l_empty_intersaction}
\{x:\Phi_{y_\ell,\nu}(x)>\e_p\}
\cap
\{x:\Phi_{y_{\ell'},\nu}(x)>\e_p\}
= \emptyset.
\end{equation}

\textbf{Deterministic algorithm.}
Let $A\in\operatorname{Alg}_{2m}(U,L^p([0,1]^d))$ be arbitrary and~$\mathbf{x}=(x_1,\dots,x_{2m})$.
Define
\[
I_{\mathbf{x}}
:=
\left\{
\ell \in \Gamma :
\forall i\in[2m],\ 
|\Phi_{y_\ell,\nu}(x_i)| \le \e_p
\right\}.
\]
By \eqref{eq:different_l_empty_intersaction}, at~most $2m$ indices violate this condition, hence
\[
|I_{\mathbf{x}}|
\ge |\Gamma|-2m.
\]
Since $k^s \ge 4m$, it follows that
\[
|I_{\mathbf{x}}|
\ge |\Gamma|\left(\frac{k^s - 2m}{k^s}\right)
\ge \frac12 |\Gamma|.
\]
For $\ell\in I_{\mathbf{x}}$, we have $Q_{\e_p}\big(\Phi_{y_\ell,\nu}(x_i)\big)=0$ for~all $i\in[2m]$ and~therefore the~algorithm cannot distinguish $\Phi_{y_\ell,\nu}$ from~$0$. Consequently,
\[
A(\Phi_{y_\ell,\nu}) = A(0).
\]
Therefore,
\begin{align*} 
\frac{1}{2|\Gamma|} \sum_{\ell \in \Gamma,\nu\in\{\pm1\}} \| \Phi_{\ell,\nu} - A(\Phi_{\ell,\nu})\|_{L^p} 
&\ge \frac{1}{2|\Gamma|} \sum_{\ell \in I_{\mathbf{x}}}\left( \| \Phi_{\ell,+1} - A(\Phi_{\ell,+1})\|_{L^p} + \| \Phi_{\ell,-1} - A(\Phi_{\ell,-1})\|_{L^p} \right) \\ 
&\ge
\frac{1}{4|I_\mathbf{x}|} \sum_{\ell \in I_{\mathbf{x}}} \left( \| \Phi_{\ell,+1} - A(0)\|_{L^p} + \| \Phi_{\ell,-1} - A(0)\|_{L^p} \right) \\ 
&\ge 
\frac{1}{2|I_\mathbf{x}|} \sum_{\ell \in I_{\mathbf{x}}} \| \Phi_{\ell,+1}\|_{L^p},
\end{align*}
where the~last inequality follows from~the~triangle inequality \[\|2\Phi_{\ell,+1} \pm A(0)\|_{L^p} \leq \|\Phi_{\ell,+1} - A(0)\|_{L^p} + \|\Phi_{\ell,+1} + A(0)\|_{L^p}\] and~the~fact that $\Phi_{\ell,+1} = -\Phi_{\ell,-1}$.
Using \eqref{eq:positive_norm_thm_M}, we obtain
\[
\|\Phi_{y_\ell,+1}\|_{L^p}
\ge
\frac{\sqrt{B^{2-4/q}c^2-1}}{B^{1-2/q}}
\left(\frac{\Omega}{2^{1+\frac2s}\sqrt{s}}\right)^{\frac{s}{p}}
m^{-\frac1p}.
\]

\textbf{Randomized setting.} Finally, we use that same trick as in~\cite[Theorem A.5]{berner2023learning}. Let $(\mathbf{A},\mathbf{m})\in\operatorname{Alg}^{MC}_m(U,L^p)$ be arbitrary on~a~probability space $(\Omega, \mathcal{F}, \mathbb P)$ and~define
\[
\Omega_0 := \{\omega : \mathbf{m}(\omega)\le 2m\},
\qquad \mathbb{P}(\Omega_0)\ge \tfrac12,
\]
where the estimate follows from Markov's inequality and~$\mathbb{E}(\mathbf{m})\le m$.
Thus, for~every $\omega\in\Omega$,
\begin{align*}
\sup_{u\in U}
\mathbb{E}\|u - A_\omega(u)\|_{L^p}
&\ge
\frac{1}{2|\Gamma|} \sum_{\ell \in \Gamma,\nu\in\{\pm1\}} \mathbb{E} \big(\| \Phi_{\ell,\nu} - A(\Phi_{\ell,\nu})\|_{L^p} \big) \\
&\ge
\mathbb{E} \left( \mathds{1}_{\Omega_0}(\omega) \ \frac{1}{2|\Gamma|} \sum_{\ell \in \Gamma,\nu\in\{\pm1\}} \| \Phi_{\ell,\nu} - A(\Phi_{\ell,\nu})\|_{L^p} \right) \\
&\ge
\mathbb{P}(\Omega_0)\cdot
\frac{\sqrt{B^{2-4/q}c^2-1}}{2B^{1-2/q}}
\left(\frac{\Omega}{2^{1+\frac2s}\sqrt{s}}\right)^{\frac{s}{p}}
m^{-\frac1p}.
\end{align*}
Since the~algorithm was arbitrary, the~proof is complete.
\end{proof}

\section*{Acknowledgements}We would like to~express our gratitude to~Samuel Lanthaler for~many inspiring discussions, which in~particular shaped the~idea to~use fixed point iterations for~designing localized bump functions.

\newpage
\printbibliography
\end{document}